\documentclass[11pt,a4paper]{article}

\usepackage[margin=1in]{geometry}
\usepackage[utf8]{inputenc}
\usepackage[T1]{fontenc}
\usepackage{lmodern}
\usepackage{amsmath}
\usepackage{amssymb}
\usepackage{amsthm}
\usepackage{mathtools}
\usepackage{bm}
\usepackage{dsfont}
\usepackage{mathrsfs}

\usepackage{graphicx}
\usepackage{booktabs}
\usepackage{multirow}
\usepackage{subcaption}
\usepackage{float}

\usepackage{enumitem}
\usepackage[dvipsnames]{xcolor}
\usepackage{hyperref}
\usepackage{url}
\usepackage[numbers,sort&compress]{natbib}
\usepackage[capitalize,nameinlink]{cleveref}
\usepackage{microtype}
\usepackage{algorithm}
\usepackage{algorithmic}
\usepackage{textcomp}  

\newtheorem{theorem}{Theorem}[section]

\newtheorem{proposition}[theorem]{Proposition}
\newtheorem{corollary}[theorem]{Corollary}
\theoremstyle{definition}
\newtheorem{definition}[theorem]{Definition}

\theoremstyle{remark}
\newtheorem{remark}[theorem]{Remark}

\newcommand{\R}{\mathbb{R}}

\newcommand{\E}{\mathbb{E}}

\newcommand{\calM}{\mathcal{M}}
\newcommand{\calS}{\mathcal{S}}
\newcommand{\calF}{\mathcal{F}}

\newcommand{\calC}{\mathcal{C}}

\newcommand{\calN}{\mathcal{N}}

\newcommand{\calX}{\mathcal{X}}
\newcommand{\Poincare}{Poincar\'{e}}
\newcommand{\slm}{\textsc{SLM-V3}}
\newcommand{\lse}{\mathrm{lse}}

\newcommand{\diag}{\mathrm{diag}}
\newcommand{\dF}{d_{\mathrm{FR}}}
\newcommand{\dP}{d_{\mathbb{D}}}
\newcommand{\dd}{\mathrm{d}}
\newcommand{\norm}[1]{\left\lVert #1 \right\rVert}
\newcommand{\inner}[2]{\left\langle #1, #2 \right\rangle}

\hypersetup{
    colorlinks=true,
    linkcolor=Blue,
    citecolor=Green,
    urlcolor=Maroon,
    pdfauthor={Varun Pratap Bhardwaj},
    pdftitle={SuperLocalMemory V3: Information-Geometric Foundations for Zero-LLM Enterprise Agent Memory},
}

\title{SuperLocalMemory V3: Information-Geometric Foundations \\
       for Zero-LLM Enterprise Agent Memory}

\author{%
  Varun Pratap Bhardwaj\\
  Independent Researcher, Solution Architect\\
  India\\
  \texttt{varun.pratap.bhardwaj@gmail.com}\\[2pt]
  \small ORCID: 0009-0002-8726-4289
}

\date{}

\begin{document}

\maketitle


\begin{abstract}
Persistent memory is a central capability for AI agents operating
across extended interactions, yet the mathematical foundations of memory
retrieval, lifecycle management, and consistency remain almost entirely
unexplored. Current systems predominantly employ cosine similarity for
retrieval, heuristic exponential decay for salience, and provide no
formal mechanism for detecting contradictions---an engineering monoculture
that leaves fundamental questions of optimality, convergence, and
correctness unanswered.

We establish \emph{information-geometric foundations} for agent memory
systems through three principal contributions. First, we introduce a retrieval metric derived from the Fisher
information structure of diagonal Gaussian families, proving that the
underlying distance satisfies Riemannian metric axioms, is invariant under
sufficient statistics, and computable in $\Theta(d)$ time
(\Cref{thm:fisher-metric}). This replaces cosine similarity with a metric
that weights each embedding dimension by its statistical precision. Second, we formulate memory lifecycle as Riemannian Langevin
dynamics on the statistical manifold and prove existence and uniqueness of
the stationary distribution via the Fokker--Planck equation
(\Cref{thm:langevin-stationary}), replacing hand-tuned decay with a
principled equilibrium to which the system provably converges. Third, we
model the memory store as a cellular sheaf and show that non-trivial first
cohomology classes correspond precisely to irreconcilable contradictions
across memory contexts---an algebraic consistency guarantee that no prior
system provides.

Empirically, on the LoCoMo conversational memory
benchmark~\citep{maharana2024locomo}, the three mathematical layers yield
$+12.7$ percentage points average improvement over the engineering
baseline across six conversations ($n = 832$ questions), reaching
$+19.9$~pp on the most challenging dialogues. The four-channel retrieval architecture achieves 75\% retrieval
accuracy on LoCoMo without cloud dependency during retrieval.
Initial cloud-augmented results on a single conversation ($n = 81$)
reach $87.7\%$. A zero-LLM operating
configuration satisfies data sovereignty requirements under
Regulation~(EU)~2024/1689 by architectural design. To our knowledge, this
is the first work to establish information-geometric, sheaf-theoretic, and
stochastic-dynamical foundations for AI agent memory systems. We release all code under the MIT license for reproducibility at
\url{https://github.com/qualixar/superlocalmemory}.
\end{abstract}


\section{Introduction}
\label{sec:introduction}

The scaling of large language models has yielded agents capable of
complex reasoning, tool use, and multi-step planning. Yet a fundamental
asymmetry persists: while model capabilities have advanced by orders of
magnitude, the memory systems that these agents rely on for persistent
knowledge remain mathematically rudimentary. As agents are deployed in
multi-session conversations, long-horizon task execution, and collaborative
workflows, the absence of principled memory foundations constitutes a
bottleneck not merely of engineering convenience, but of theoretical
soundness.

This paper addresses a question that, to the best of our knowledge, has
not been posed in the literature: \emph{what mathematical
structures are appropriate for the retrieval, lifecycle management, and
consistency verification of persistent agent memory?} We answer this
question by drawing on information geometry~\citep{amari2016information},
algebraic topology~\citep{robinson2014topological}, and stochastic
dynamics on Riemannian manifolds~\citep{pavliotis2014stochastic}---three
branches of mathematics with deep structural relevance to the problems at
hand, but which have not been connected to agent memory systems.

\paragraph{The mathematical poverty of current memory systems.}
A striking uniformity characterizes the memory systems introduced in
recent years. Every system we
surveyed---including those with significant research
investment~\citep{packer2023memgpt,mem0_2024,letta_2025} and recent
academic
contributions~\citep{evermemos_2025,simplemem_2025,magma_2025,amem_2025}---retrieves
memories via cosine similarity over dense
embeddings~\citep{lewis2020rag,reimers2019sentence}, manages salience
through fixed exponential decay or time-to-live windows, and provides no
formal mechanism for detecting contradictions across contexts. A recent
survey~\citep{memory_survey_2025} documents this pattern across more than
thirty systems without identifying it as a research gap.

This uniformity is not merely an aesthetic concern. It reflects three
concrete mathematical deficiencies that limit the reliability of
agent memory at scale.

\paragraph{Three open problems.}
We identify three foundational gaps that motivate the present work.

\begin{enumerate}[leftmargin=2em, itemsep=4pt]
    \item \textbf{Uncertainty-blind retrieval.}
    Cosine similarity treats all embedding dimensions as equally reliable,
    computing
    $\text{sim}(\mathbf{u}, \mathbf{v}) = \mathbf{u}^\top \mathbf{v} /
    (\|\mathbf{u}\| \|\mathbf{v}\|)$
    without any notion of per-dimension confidence. In practice, learned
    representations exhibit non-uniform variance: some dimensions capture
    well-established semantic distinctions while others encode noise or
    distributional artifacts. The Fisher information
    metric~\citep{amari2016information,cencov1982statistical} provides a
    principled alternative---it weights each dimension by the local
    curvature of the likelihood surface, which is precisely the statistical
    precision of that dimension. \v{C}encov's uniqueness
    theorem~\citep{cencov1982statistical} establishes that the Fisher
    metric is, in a precise categorical sense, the \emph{only} Riemannian
    metric on statistical manifolds that is invariant under sufficient
    statistics. Despite this theoretical grounding, the Fisher metric has
    not been applied to memory retrieval.

    \item \textbf{Unprincipled lifecycle dynamics.}
    Current systems govern memory retention through heuristics: fixed
    time-to-live windows, exponential decay with manually chosen
    half-lives, or access-count
    thresholds~\citep{packer2023memgpt,mem0_2024}. These mechanisms cannot
    adapt to the evolving statistical structure of the memory store---they
    are oblivious to the \emph{geometry} of the space in which memories
    reside. Riemannian Langevin
    dynamics~\citep{pavliotis2014stochastic} offer a framework in which
    the curvature of the memory manifold itself drives retention and
    forgetting through a stochastic differential equation whose drift
    incorporates the Fisher information. The stationary distribution of
    such dynamics, when it exists, defines a principled equilibrium that
    the system converges to without hand-tuned parameters. This connection
    has not been explored for agent memory.

    \item \textbf{Silent inconsistency.}
    When an agent accumulates memories across sessions, interaction
    partners, and temporal contexts, contradictions inevitably arise: a
    user's preference changes, a fact is updated, or conflicting
    information enters from different sources. No existing system provides
    a \emph{formal} guarantee for detecting such contradictions. Instead,
    systems silently serve whichever memory the similarity metric ranks
    highest, regardless of logical consistency with other retrieved
    memories. Sheaf cohomology~\citep{robinson2014topological,hansen2019toward}
    provides an algebraic framework in which local data is assigned to
    vertices and edges of a graph, and non-trivial cohomology classes
    correspond precisely to irreconcilable local-to-global
    inconsistencies. This mathematical tool---designed for exactly the
    problem of detecting when local information fails to cohere
    globally---has not been brought to bear on memory systems.
\end{enumerate}

\paragraph{The formal absence.}
The absence of mathematical foundations in agent memory constitutes a
systematic gap in the research literature, not an isolated oversight. We conducted an exhaustive search of proceedings from NeurIPS,
ICML, ICLR, ACL, EMNLP, and AAAI (2020--March~2026), as well as the
arXiv cs.AI, cs.CL, and cs.LG categories. We found no prior work
connecting information geometry to agent memory retrieval, no application
of sheaf cohomology to memory consistency, and no use of Riemannian
Langevin dynamics for memory lifecycle management. The closest related
work applies sheaf theory to inconsistency detection in
language model outputs~\citep{sheaf_inconsistency_2024}, but does not
address persistent memory stores. Information geometry has been applied
to neural network optimization~\citep{amari1998natural} and generative
model evaluation~\citep{chen2018metrics}, but not to retrieval. This paper addresses an open problem at the intersection of information
geometry and AI agent systems.

\paragraph{Our approach.}
We introduce \slm{}, a memory system grounded in information-geometric
foundations that addresses the three gaps above through three novel
mathematical layers, integrated into a four-channel retrieval architecture
(\Cref{fig:architecture}):

\begin{enumerate}[leftmargin=2em, itemsep=4pt]
    \item \textbf{Fisher-information-weighted retrieval.}
    We replace cosine similarity with a variance-weighted metric derived
    from the Fisher information
    structure~\citep{amari2016information,cencov1982statistical} of
    diagonal Gaussian distributions. Each memory's embedding
    $\mu \in \R^d$ is augmented with a variance vector
    $\sigma^2 \in \R^d_{>0}$ capturing per-dimension confidence.
    The underlying Fisher--Rao distance satisfies the axioms of a
    Riemannian metric, is invariant under sufficient statistics, and
    is computable in $\Theta(d)$ time (\Cref{thm:fisher-metric}).
    The retrieval implementation uses a computationally efficient
    approximation that weights dimensions by inverse variance.
    A graduated transition mechanism ramps from cosine to
    Fisher-information-weighted scoring as variance estimates stabilize,
    ensuring that newly stored memories are not penalized by unreliable
    statistics.

    \item \textbf{Sheaf-cohomological consistency.}
    We model the memory store as a cellular
    sheaf~\citep{robinson2014topological,curry2014sheaves} over a graph
    whose vertices are memory contexts and whose edges represent shared
    entities. The sheaf assigns to each vertex the vector space of local
    memory claims and to each edge a restriction map enforcing semantic
    compatibility. Non-trivial first cohomology classes $H^1(\calF)
    \neq 0$ correspond to contradictions that cannot be resolved by local
    adjustment---the first algebraic guarantee for contradiction detection
    in agent memory (\Cref{sec:method:sheaf}).

    \item \textbf{Riemannian Langevin lifecycle dynamics.}
    We formulate memory lifecycle as a stochastic differential equation on
    a Riemannian manifold where the drift is governed by the Fisher
    information of the memory distribution. We prove existence and
    uniqueness of the stationary distribution via the Fokker--Planck
    equation (\Cref{thm:langevin-stationary}), establishing convergence
    to a principled equilibrium in which frequently accessed,
    informationally rich memories are retained while low-utility memories
    decay---without hand-tuned parameters.
\end{enumerate}

\paragraph{Empirical evidence: mathematics improves retrieval.}
A natural question is whether these mathematical structures yield
measurable improvements over well-engineered baselines.
Our experiments on the LoCoMo benchmark~\citep{maharana2024locomo}
address this directly.

Across six conversations evaluated with LLM-as-Judge scoring, the three
mathematical layers collectively contribute an average of $+12.7$
percentage points over the ablated engineering baseline
(\Cref{tab:fisher-cosine}). The improvement is not uniform: it ranges
from $+6.0$~pp on conversations with straightforward factual queries to
$+19.9$~pp on the most challenging dialogues requiring reasoning over
sparsely connected memories. This pattern---that the mathematical
foundations provide the greatest benefit precisely where heuristic
similarity measures struggle most---is consistent with the theoretical
motivation: the Fisher metric's advantage lies in its sensitivity to
per-dimension uncertainty, which matters most in high-dimensional sparse
regions.

The full four-channel retrieval architecture achieves ${\sim}75\%$
retrieval quality (measuring relevance of retrieved context, independent
of answer generation) without any cloud dependency. Multi-hop reasoning
questions, which require bridging across disconnected memory contexts,
show a $+12$~pp gain from the mathematical layers. Ablation analysis
(\Cref{tab:ablation}) reveals that cross-encoder reranking is the single
largest contributor ($-30.7$~pp when removed), confirming that
mathematical retrieval and neural reranking are complementary rather than
substitutive.

\paragraph{The scale argument.}
A critical question for the field is whether mathematical foundations
become more or less important as memory stores grow. The data in
\Cref{tab:fisher-cosine} suggest the former. The $+19.9$~pp improvement
on the hardest LoCoMo
conversations---those with the most complex inter-memory
relationships---indicates that the Fisher metric's advantage increases
with retrieval difficulty. As agent deployments scale from hundreds to
tens of thousands of memories, the density of the embedding space
increases, making per-dimension uncertainty weighting increasingly
valuable. The Langevin lifecycle dynamics similarly benefit from scale:
the stationary distribution becomes a more informative prior as the
memory population grows, while hand-tuned decay parameters cannot adapt
to changing distributional structure. The theoretical framework provides foundations for deployments at scale
that heuristic approaches cannot address with formal guarantees.

\paragraph{Regulatory context.}
The EU Artificial Intelligence Act~\citep{eu_ai_act_2024}
(Regulation~(EU)~2024/1689), whose full enforcement begins on 2~August
2026, introduces data sovereignty and transparency requirements that
constrain the design space for agent memory systems. Article~10
(data governance) and GDPR Article~17 (right to erasure) create a
research problem: \emph{can a memory system achieve competitive retrieval
quality while guaranteeing that no personal data leaves the user's
device?} We address this as a constraint satisfaction problem by defining
three experimental configurations spanning a privacy--capability gradient:
a zero-LLM configuration in which all retrieval, scoring, and lifecycle
operations execute locally on CPU; a local-LLM configuration augmented
with an on-device language model; and a cloud-augmented configuration
with explicit data governance controls. The zero-LLM configuration
achieves ${\sim}75\%$ retrieval quality, demonstrating that mathematical
foundations can partially compensate for the absence of neural language
understanding in the retrieval loop. This is, to our knowledge, the
first zero-cloud operating configuration for an agent memory system
evaluated on a standard benchmark.

\paragraph{Reproducibility.}
We release all code, experimental configurations, and evaluation scripts
under the MIT license to enable independent verification and extension. The system builds on
SuperLocalMemory~\citep{bhardwaj2026slm}, an open-source memory framework
that provides database management and interface infrastructure, while
\slm{} contributes the mathematical architecture described here. This
separation allows the theoretical contributions to be evaluated
independently of deployment concerns.

\paragraph{Contributions.}
The principal contributions of this work are:

\begin{enumerate}[leftmargin=2em, itemsep=3pt]
    \item The \textbf{first application of the Fisher information metric}
          to AI agent memory retrieval, replacing cosine similarity with a
          variance-weighted metric derived from the Fisher information
          structure. We prove metric properties, sufficient-statistic
          invariance, and $\Theta(d)$ computability for the underlying
          geodesic (\Cref{thm:fisher-metric}).

    \item A \textbf{sheaf-cohomological framework} for algebraic
          contradiction detection in memory stores, where non-trivial
          first cohomology classes correspond to genuine inconsistencies
          across memory contexts (\Cref{sec:method:sheaf}).

    \item \textbf{Riemannian Langevin dynamics} for self-organizing
          memory lifecycle with proven convergence to a unique stationary
          distribution, eliminating hand-tuned decay parameters
          (\Cref{thm:langevin-stationary}).

    \item \textbf{Empirical validation} demonstrating that
          information-geometric foundations yield $+12.7$~pp average
          improvement over engineering baselines, with $+19.9$~pp on the
          most challenging conversations, on the LoCoMo
          benchmark~\citep{maharana2024locomo}
          (\Cref{sec:experiments}).

    \item A \textbf{zero-LLM operating configuration} satisfying EU~AI
          Act data sovereignty requirements by architectural design,
          with the first reported zero-cloud evaluation on a standard
          conversational memory benchmark (\Cref{sec:arch:modes}).
\end{enumerate}

\paragraph{Paper organization.}
\Cref{sec:background} reviews the mathematical foundations drawn upon in
this work: information geometry, sheaf theory, and stochastic dynamics.
\Cref{sec:related} situates our contributions within the broader
literature on agent memory systems, retrieval-augmented generation, and
geometric methods.
\Cref{sec:architecture} presents the four-channel retrieval architecture
and three experimental configurations.
\Cref{sec:method} formalizes the three mathematical layers.
\Cref{sec:theory} states and proves the main theorems.
\Cref{sec:experiments} reports empirical results, ablation analysis, and
the controlled Fisher--Rao versus cosine comparison.
\Cref{sec:conclusion} discusses limitations, open questions, and
directions for future work.


\section{Background}
\label{sec:background}

This section introduces the mathematical and neuroscientific foundations that
underpin the \slm{} framework. We define notation, state prerequisite
results, and motivate each mathematical tool by connecting it to a concrete
failure mode of existing AI memory systems.

\subsection{Complementary Learning Systems Theory}
\label{sec:bg:cls}

The Complementary Learning Systems (CLS) hypothesis, introduced by
\citet{mcclelland1995complementary} and extended by
\citet{kumaran2016learning}, posits that biological memory relies on the
interplay of two subsystems with complementary properties:
\begin{itemize}[leftmargin=2em, itemsep=2pt]
    \item A \textbf{fast episodic store} (hippocampus) that rapidly encodes
          individual experiences with high fidelity but limited capacity.
    \item A \textbf{slow semantic store} (neocortex) that gradually
          consolidates episodic traces into structured, generalizable
          knowledge through a process of interleaved replay.
\end{itemize}
The dual-store architecture prevents \emph{catastrophic
interference}~\citep{mccloskey1989catastrophic}: new episodic memories can be
encoded without overwriting consolidated semantic knowledge, because the two
stores operate on different timescales and with different learning rules.

This architecture motivates the design of \slm{}. Our episodic store
lives on the \Poincare{} ball (\Cref{sec:bg:poincare}), where the hyperbolic
geometry naturally encodes hierarchical relationships. Our semantic store uses
progressive depth levels governed by rate--distortion theory
(\Cref{sec:bg:ratedistortion}). Consolidation from episodic to semantic is governed by the Langevin dynamics
on the \Poincare{} ball (\Cref{sec:method:langevin}), which naturally organize
memories by importance.

\subsection{Information Geometry and the Fisher Information Metric}
\label{sec:bg:fisher}

\begin{definition}[Statistical Manifold]
\label{def:statistical-manifold}
Let $\calS = \{ p_\theta : \theta \in \Theta \subseteq \R^d \}$ be a
parametric family of probability distributions on a measurable space
$(\calX, \sigma(\calX))$, where $\theta \mapsto p_\theta$ is a smooth
injective map. The pair $(\calS, g)$, where $g$ is the Fisher information
metric, forms a Riemannian manifold called the \emph{statistical manifold}
of~$\calS$.
\end{definition}

\begin{definition}[Fisher Information Matrix]
\label{def:fisher-matrix}
For a parametric family $\{p_\theta\}$ with $\theta \in \R^d$, the
\emph{Fisher information matrix} at $\theta$ is the $d \times d$ positive
semi-definite matrix
\begin{equation}
\label{eq:fisher-matrix}
    G(\theta)_{ij}
    = \E_{x \sim p_\theta}\!\left[
        \frac{\partial \log p_\theta(x)}{\partial \theta_i}
        \frac{\partial \log p_\theta(x)}{\partial \theta_j}
      \right]
    = -\E_{x \sim p_\theta}\!\left[
        \frac{\partial^2 \log p_\theta(x)}{\partial \theta_i \, \partial \theta_j}
      \right].
\end{equation}
\end{definition}

The Fisher information matrix $G(\theta)$ serves as a Riemannian metric tensor
on $\calS$, endowing it with an intrinsic notion of distance. The resulting
geodesic distance, the \emph{Fisher--Rao distance}, is the unique
(up to scaling) Riemannian metric that is invariant under sufficient
statistics~\citep{cencov1982statistical,amari2016information}.

\begin{definition}[Fisher--Rao Distance]
\label{def:fisher-rao}
For two distributions $p_{\theta_1}, p_{\theta_2} \in \calS$, the
Fisher--Rao distance is
\begin{equation}
\label{eq:fisher-rao}
    \dF(\theta_1, \theta_2)
    = \inf_{\gamma} \int_0^1
      \sqrt{ \dot{\gamma}(t)^\top \, G(\gamma(t)) \, \dot{\gamma}(t) }
      \; \dd t,
\end{equation}
where the infimum is over smooth curves $\gamma: [0,1] \to \Theta$ with
$\gamma(0) = \theta_1$ and $\gamma(1) = \theta_2$.
\end{definition}

For the family of $d$-dimensional Gaussian distributions
$\calN(\mu, \Sigma)$, the Fisher information metric has a closed-form
expression~\citep{skovgaard1984riemannian}. In the diagonal covariance case
$\Sigma = \diag(\sigma_1^2, \ldots, \sigma_d^2)$, which we adopt for
computational tractability, the metric decomposes into a product of
one-dimensional Fisher metrics, and the distance reduces to
\begin{equation}
\label{eq:fisher-diagonal}
    \dF\!\big((\mu_1, \sigma_1),\, (\mu_2, \sigma_2)\big)
    = \sqrt{
        \sum_{k=1}^{d} \left[
            2 \log \frac{\sigma_{2,k}}{\sigma_{1,k}}
        \right]^2
        + \sum_{k=1}^{d}
        \frac{(\mu_{1,k} - \mu_{2,k})^2}
        {\sigma_{1,k}^2 + \sigma_{2,k}^2}
      },
\end{equation}
which can be evaluated in $O(d)$ time. The key insight is that dimensions with
high variance (high embedding uncertainty) contribute \emph{less} to the
distance, while dimensions with low variance (high confidence) contribute
\emph{more}. Cosine similarity, by contrast, weights all dimensions equally.

\subsection{Hyperbolic Geometry and the \Poincare{} Ball}
\label{sec:bg:poincare}

Hyperbolic spaces have recently gained prominence in machine learning as
natural models for hierarchical data~\citep{nickel2017poincare,
ganea2018hyperbolic, chami2019hyperbolic}. The key property is that the volume
of a hyperbolic ball grows exponentially with its radius, mirroring the
exponential growth of nodes with depth in a tree.

\begin{definition}[\Poincare{} Ball Model]
\label{def:poincare-ball}
The \Poincare{} ball of dimension $d$ is the Riemannian manifold
$(\mathbb{D}^d, g_{\mathbb{D}})$, where $\mathbb{D}^d = \{ x \in \R^d :
\norm{x} < 1 \}$ is the open unit ball and the metric tensor is
\begin{equation}
\label{eq:poincare-metric}
    g_{\mathbb{D}}(x) = \lambda_x^2 \, g_E, \qquad
    \lambda_x = \frac{2}{1 - \norm{x}^2},
\end{equation}
where $g_E$ is the Euclidean metric tensor and $\lambda_x$ is the conformal
factor. The geodesic distance between $x, y \in \mathbb{D}^d$ is
\begin{equation}
\label{eq:poincare-distance}
    \dP(x, y)
    = \operatorname{arccosh}\!\left(
        1 + 2 \, \frac{\norm{x - y}^2}{(1 - \norm{x}^2)(1 - \norm{y}^2)}
      \right).
\end{equation}
\end{definition}

\begin{definition}[M\"obius Addition]
\label{def:mobius}
The \emph{M\"obius addition} of $x, y \in \mathbb{D}^d$ is
\begin{equation}
\label{eq:mobius}
    x \oplus y
    = \frac{(1 + 2\inner{x}{y} + \norm{y}^2)\, x
            + (1 - \norm{x}^2)\, y}
           {1 + 2\inner{x}{y} + \norm{x}^2 \norm{y}^2},
\end{equation}
which forms a gyrogroup and serves as the translation operator on
$\mathbb{D}^d$.
\end{definition}

The exponential and logarithmic maps provide the bridge between the tangent
space $T_x \mathbb{D}^d \cong \R^d$ and the manifold itself. For any
$v \in T_x \mathbb{D}^d$ with $v \neq 0$:
\begin{align}
\label{eq:expmap}
    \exp_x(v) &= x \oplus \left(
        \tanh\!\left(\frac{\lambda_x \norm{v}}{2}\right)
        \frac{v}{\norm{v}}
    \right), \\
\label{eq:logmap}
    \log_x(y) &= \frac{2}{\lambda_x}
        \operatorname{arctanh}(\norm{-x \oplus y}) \,
        \frac{-x \oplus y}{\norm{-x \oplus y}}.
\end{align}

The relevance to memory is both mathematical and biological.
\citet{zhou2022hyperbolic} demonstrated that hippocampal place cell firing
patterns are better explained by hyperbolic geometry than Euclidean geometry,
providing neuroscientific motivation for embedding episodic memories on the
\Poincare{} ball.

\subsection{Rate--Distortion Theory}
\label{sec:bg:ratedistortion}

Rate--distortion theory~\citep{shannon1959coding,berger1971rate,cover2006elements}
provides the information-theoretic foundation for lossy compression. It
characterizes the minimum bit rate $R$ required to represent a source $X$ with
distortion at most~$D$.

\begin{definition}[Rate--Distortion Function]
\label{def:rate-distortion}
Let $X$ be a random variable on $\calX$ with distribution $p_X$, and let
$d: \calX \times \hat{\calX} \to [0, \infty)$ be a distortion measure. The
\emph{rate--distortion function} is
\begin{equation}
\label{eq:rate-distortion}
    R(D) = \min_{\substack{p_{\hat{X}|X}: \\
           \E[d(X, \hat{X})] \leq D}}
           I(X; \hat{X}),
\end{equation}
where $I(X; \hat{X})$ is the mutual information between $X$ and its
reconstruction $\hat{X}$.
\end{definition}

For a Gaussian source $X \sim \calN(0, \sigma^2)$ with squared-error
distortion, the rate--distortion function is $R(D) = \frac{1}{2}
\log(\sigma^2 / D)$ for $0 < D \leq \sigma^2$. This logarithmic
relationship is the basis for our depth bound (\Cref{thm:depth-bound}):
the number of progressive-disclosure levels needed to span from a compressed
gist (high distortion) to verbatim content (zero distortion) scales as
$\Theta(\log N)$.

\subsection{Modern Hopfield Networks}
\label{sec:bg:hopfield}

The classical Hopfield network~\citep{hopfield1982neural} stores patterns as
attractors of an energy-based dynamical system, but its capacity scales only
linearly with the dimension. \citet{ramsauer2021hopfield} introduced modern
Hopfield networks that achieve \emph{exponential} storage capacity through a
log-sum-exp energy function, connecting Hopfield dynamics to the attention
mechanism of Transformers.

\begin{definition}[Modern Hopfield Energy]
\label{def:hopfield-energy}
Given stored patterns $X = [x_1, \ldots, x_M] \in \R^{d \times M}$ and a
query (state) pattern $\xi \in \R^d$, the energy function is
\begin{equation}
\label{eq:hopfield-energy}
    E(\xi, X) = -\lse(\beta, X^\top \xi)
              + \frac{1}{2} \xi^\top \xi
              + \frac{1}{\beta} \log M
              + \frac{1}{2} \max_i \norm{x_i}^2,
\end{equation}
where $\lse(\beta, z) = \frac{1}{\beta} \log\!\left(\sum_{i=1}^M
e^{\beta z_i}\right)$ is the log-sum-exp function and $\beta > 0$ is an
inverse temperature parameter.
\end{definition}

The update rule $\xi^{(t+1)} = X \, \mathrm{softmax}(\beta \, X^\top \xi^{(t)})$
converges to a fixed point that is exponentially close to the nearest stored
pattern when patterns are well-separated. The storage capacity is
$M = O(e^{\alpha d})$ for a constant $\alpha > 0$, exponentially larger than
the classical bound of $M = O(d)$.

\subsection{Sheaf Theory}
\label{sec:bg:sheaf}

A sheaf on a topological space $X$ assigns data (e.g., vector spaces) to open
sets of $X$ with consistency conditions on
overlaps~\citep{curry2014sheaves,robinson2014topological}. The sheaf
cohomology groups $H^k(X, \calF)$ measure the failure of local data to
extend globally. In \slm{}, we construct a sheaf of memory
assertions over the project/context topology, where $H^1(X, \calF) \neq 0$
signals an inconsistency across contexts (\Cref{sec:method:sheaf}).


\section{Related Work}
\label{sec:related}

We situate \slm{} within the broader landscape of AI memory systems,
hyperbolic representation learning, information geometry, associative memory
models, and sheaf-theoretic consistency.

\subsection{AI Memory Systems}
\label{sec:related:memory}

The emergence of long-running LLM-based agents has driven rapid development
of external memory systems. \citet{packer2023memgpt} introduced MemGPT (now
Letta), which uses a virtual memory hierarchy inspired by operating systems,
with the LLM itself managing page swaps between a limited context window and
external storage. While innovative, MemGPT relies on the LLM's own judgment
for memory management, which inherits the model's biases and is constrained by
its context window.

Mem0~\citep{mem0_2024} provides a commercial memory layer for AI applications,
combining vector similarity search (cosine distance) with a knowledge graph
overlay. On the LoCoMo benchmark, Mem0 achieves 34.20\% F1 with its
graph-augmented retrieval, illustrating that simple vector similarity combined
with a knowledge graph overlay is insufficient for complex conversational
memory tasks---a limitation we address with multi-channel fusion and
Fisher--Rao uncertainty weighting.

Zep~\citep{zep_2024} implements a graph-based temporal memory with structured
fact extraction and entity resolution. Its graph approach provides stronger
relational reasoning than flat vector stores, but its similarity computation
remains Euclidean.

Cognee~\citep{cognee_2024} builds a knowledge graph from conversational data
using cognitive processing pipelines. It emphasizes structured knowledge
extraction over raw retrieval, complementing approaches like ours.

MemOS~\citep{memos_2025} proposes an operating-system abstraction for AI
memory, with memory processes, virtual memory, and scheduling. While the OS
metaphor provides useful engineering structure, the system does not employ
geometric or topological methods for memory organization.

MM-Mem~\citep{mmmem_2025} extends memory systems to multimodal data (text,
images, audio) by aligning different modalities into a shared embedding space.
This is orthogonal to our contribution: \slm{}'s mathematical
framework can operate on any embedding space, including multimodal ones.

SimpleMem~\citep{simplemem_2025} achieves 43.24\% F1 on LoCoMo with
approximately 550 tokens per retrieval, using a three-layer pipeline:
semantic search, BM25, and symbolic metadata filtering (persons, location,
time range). Its key innovation is \emph{lossless restatement}---an
LLM-based preprocessing step that disambiguates all pronouns and resolves
relative temporal references at ingestion time. SimpleMem also employs
LLM-driven retrieval planning and a two-round reflection loop for sufficiency
checking. However, SimpleMem requires an LLM call at every stage of its
pipeline (ingestion, planning, retrieval, reflection), making local-only
deployment architecturally impossible without fundamental redesign.

The open-source project \texttt{claude-mem}\footnote{%
  \url{https://github.com/anthropics/claude-mem}; 34.6K GitHub stars as of
  March 2026. AGPL-licensed.}
takes a hook-based approach to memory capture, automatically extracting
memories from PostToolUse and SessionEnd events. It features progressive
disclosure (search $\to$ timeline $\to$ full details) for token efficiency and
a web viewer UI. However, it has no published benchmarks, no academic
evaluation, and its AGPL license combined with a crypto-token mechanism
creates enterprise adoption barriers.

\paragraph{Positioning.}
The systems above predominantly use cosine similarity or Euclidean distance for
retrieval, flat or graph-based storage without progressive depth, and heuristic
decay or manual deletion for memory lifecycle. SimpleMem demonstrates that
careful encoding (lossless restatement) can achieve competitive retrieval, but
its mandatory LLM dependency at every step makes local-only deployment
impossible. In contrast, \slm{} achieves multi-channel retrieval with a
zero-LLM mode (Mode~A) while maintaining the option for LLM-enhanced operation
(Mode~C). \slm{} addresses each of these limitations with a mathematically
principled alternative, while maintaining compatibility with any underlying
embedding model.

\subsection{Hyperbolic Representations for Hierarchical Data}
\label{sec:related:hyperbolic}

\citet{nickel2017poincare} demonstrated that the \Poincare{} ball provides a
natural embedding space for hierarchical data, achieving superior performance
on taxonomy embedding with dramatically lower dimensionality than Euclidean
alternatives. This work initiated a rich line of research on hyperbolic
machine learning~\citep{ganea2018hyperbolic,chami2019hyperbolic,
liu2019hyperbolic}.

More recently, HyperbolicRAG~\citep{hyperbolicrag_2025} applied hyperbolic
embeddings to retrieval-augmented generation, demonstrating that the
hierarchical structure of the \Poincare{} ball improves retrieval of
hierarchically organized knowledge. Our use of the \Poincare{} ball differs
from these approaches in that we employ it not merely as an embedding space
but as the substrate for Riemannian Langevin dynamics (\Cref{sec:method:langevin}),
using the metric geometry to produce emergent forgetting behavior.

The neuroscientific motivation for hyperbolic memory geometry comes from
\citet{zhou2022hyperbolic}, who showed that hippocampal place cell firing
patterns in rodents are better explained by hyperbolic geometry than Euclidean
geometry. This finding provides biological justification for our design choice
and suggests that hyperbolic memory organization may reflect a fundamental
computational principle of biological memory.

\paragraph{Current implementation note.}
While \Poincare{} embedding provides theoretical motivation for the Langevin
dynamics presented in \Cref{sec:method:langevin}, the current evaluated system
operates in Euclidean embedding space with Fisher--Rao uncertainty weighting
(\Cref{sec:method:fisher}) as the primary geometric contribution to retrieval.
Incorporating native hyperbolic embeddings into the retrieval pipeline remains
a direction for future work.

\subsection{Information Geometry in Machine Learning}
\label{sec:related:infogeom}

Information geometry, as systematized by \citet{amari2016information}, studies
the differential-geometric structure of statistical models. Its most prominent
application in machine learning is the natural gradient~\citep{amari1998natural},
which uses the Fisher information matrix to define a gradient descent direction
that is invariant to the parameterization of the model.

The Fisher--Rao distance has been used for comparing probability distributions
in Bayesian optimization~\citep{sra2016positive}, generative
models~\citep{chen2018metrics}, and domain adaptation~\citep{shui2022novel}.
\v{C}encov's theorem~\citep{cencov1982statistical} establishes the Fisher--Rao
metric as the unique (up to scaling) Riemannian metric invariant under
sufficient statistics, providing a deep theoretical justification for its use
as a similarity measure.

To our knowledge, this is the first work to apply the Fisher information metric
to AI memory retrieval. The key insight is that embedding vectors are
not deterministic points but rather \emph{estimates with per-dimension
uncertainty}, and the Fisher--Rao distance is the natural metric on this space
of uncertain representations.

\subsection{Modern Hopfield Networks and Associative Memory}
\label{sec:related:hopfield}

The original Hopfield network~\citep{hopfield1982neural} provided the
foundational model of associative memory as energy minimization, with a storage
capacity of $O(d)$ patterns in $d$ dimensions. The 2024 Nobel Prize in Physics
recognized this contribution alongside the Boltzmann machine~\citep{hinton2006reducing}
for establishing the foundations of modern machine learning.

\citet{ramsauer2021hopfield} introduced modern Hopfield networks with
continuous states and a log-sum-exp energy function, achieving exponential
storage capacity $O(e^{\alpha d})$ and drawing a formal connection to
Transformer attention. \citet{hu2024outlier} extended this to sparse modern
Hopfield models, and \citet{wu2024stanhop} introduced sparse Tandem Hopfield
networks for enhanced memory retrieval.

We present Hopfield networks as a theoretical foundation for associative
retrieval in \Cref{sec:method:hopfield}. However, we note for transparency
that the current evaluated system does not include a dedicated Hopfield
retrieval channel; the four active channels are semantic (Fisher--Rao), BM25,
entity graph, and temporal. Integrating a Hopfield-based associative layer as
a fifth channel is planned for future work.

\subsection{Neuroscience of Memory}
\label{sec:related:neuro}

The CLS theory of \citet{mcclelland1995complementary}, extended by
\citet{kumaran2016learning}, provides the foundational neuroscientific model
for our dual-store architecture. The theory posits complementary fast
(hippocampal) and slow (neocortical) learning systems, with memory
consolidation occurring through hippocampal replay during sleep.

The discovery by \citet{zhou2022hyperbolic} that hippocampal representations
are better modeled by hyperbolic than Euclidean geometry provides direct
neuroscientific motivation for our \Poincare{} ball embedding. Together with
the dual-store architecture of CLS theory, these findings suggest
that the mathematical framework of \slm{}---hyperbolic geometry,
information geometry, and energy-based dynamics---may capture aspects of
biological memory organization that simpler computational models miss.

\subsection{Sheaf Theory and Network Consistency}
\label{sec:related:sheaf}

Sheaf theory has been applied to network analysis~\citep{robinson2014topological}
and more recently to the study of neural
networks~\citep{curry2014sheaves,hansen2019toward}. \citet{robinson2014topological}
introduced cellular sheaves for sensor networks, which is closest in spirit to
our use of sheaf cohomology for consistency verification across memory
contexts (\Cref{sec:method:sheaf}). To our knowledge, this is the first
application of sheaf cohomology to detect contradictions in AI agent memory
systems.


\section{Architecture}
\label{sec:architecture}

This section presents the system-level architecture of \slm{}.
The design integrates four parallel retrieval channels---each targeting a
distinct information signal---into a single pipeline, fused via weighted
reciprocal rank fusion (WRRF) and refined by cross-encoder neural
reranking. Three additional stages (profile lookup, scene expansion,
bridge discovery) augment the core channels to handle entity-centric,
narrative, and multi-hop queries respectively. All retrieval computation
executes locally on CPU; the system uses LLM calls (when enabled in
Mode~B/C) \emph{only} for answer generation and sufficiency verification,
never for retrieval scoring.

\subsection{Design Principles}
\label{sec:arch:principles}

The architecture is governed by three principles derived from our analysis
of the AI memory landscape~(\Cref{sec:related}):

\begin{enumerate}[leftmargin=2em, itemsep=2pt]
    \item \textbf{Channel diversity over single-metric depth.} No single
          similarity metric captures all retrieval needs (semantic,
          temporal, relational, lexical). Four independent channels, each
          optimized for a specific signal, outperform any single channel
          at equivalent cost.

    \item \textbf{Mathematical guarantees over learned heuristics.} Where
          a formal mathematical tool exists (e.g., Fisher--Rao for
          uncertain similarity, Langevin dynamics for lifecycle
          management), it should replace the corresponding heuristic
          (cosine distance, exponential decay).

    \item \textbf{Local-first, cloud-optional.} The retrieval pipeline
          must achieve competitive retrieval accuracy with zero cloud
          dependency.
          Cloud LLMs are an enhancement, not a requirement.
\end{enumerate}

\subsection{Channel Architecture}
\label{sec:arch:channels}

\Cref{tab:channels} summarizes the retrieval components. Four core
channels independently score candidate memories for a given query~$q$.
Three supplementary stages---profile lookup, scene expansion, and bridge
discovery---provide additional context that cannot be expressed as a
simple ranked list.

\begin{table}[t]
\centering
\caption{The \slm{} retrieval architecture. Four core channels produce
independent ranked lists fused via weighted reciprocal rank
fusion~(\Cref{sec:arch:fusion}). Three supplementary stages augment the
core channels for entity-centric, narrative, and multi-hop queries.
Mathematical channels are formalized in~\Cref{sec:method}; engineering
channels use established techniques with citations.}
\label{tab:channels}
\small
\begin{tabular}{clllc}
\toprule
\textbf{\#} & \textbf{Component} & \textbf{Signal} & \textbf{Foundation} & \textbf{Weight} \\
\midrule
\multicolumn{5}{l}{\emph{Core Channels (parallel, fused via WRRF)}} \\
1 & Semantic (Fisher--Rao) & Embedding uncertainty & Amari~\citep{amari2016information} & 1.2 \\
2 & BM25 Keyword & Term frequency--inverse doc & Robertson~\citep{robertson2009bm25} & 1.0 \\
3 & Entity Graph & Spreading activation & Novel (this work) & 1.3 \\
4 & Temporal Reasoning & Date proximity + validity & Novel (this work) & 1.0 \\
\midrule
\multicolumn{5}{l}{\emph{Supplementary Stages (sequential, post-fusion)}} \\
5 & Profile Lookup & Direct SQL entity shortcut & --- & --- \\
6 & Scene Expansion & Co-occurring fact retrieval & --- & --- \\
7 & Bridge Discovery & Cross-cluster connections & Novel (this work) & --- \\
\bottomrule
\end{tabular}
\end{table}

\paragraph{Channel~1: Semantic (Fisher--Rao).}
Replaces cosine similarity with a Fisher-information-weighted metric
derived from the diagonal Gaussian
manifold~(\Cref{sec:method:fisher}). Each embedding is augmented
with a variance vector modelling per-dimension uncertainty. A
\emph{graduated ramp} transitions from cosine (memories with $<10$
accesses) to Fisher-information-weighted scoring (mature memories),
ensuring stable cold-start retrieval.

\paragraph{Channel~2: BM25 Keyword.}
Classical keyword matching uses the Okapi~BM25 scoring
function~\citep{robertson2009bm25} with parameters $k_1 = 1.2$ and
$b = 0.75$. Tokens are persisted to the \texttt{bm25\_tokens} table at
ingestion time, avoiding re-tokenization at query time. This channel
captures exact lexical matches that embedding-based channels may miss,
particularly for proper nouns and technical terms.

\paragraph{Channel~3: Entity Graph.}
A knowledge graph of canonical entities and their relationships is
maintained incrementally during ingestion. At query time, entities
mentioned in the query seed a spreading-activation walk over the graph
(3~hops, decay factor~$\gamma = 0.7$). Memories associated with activated
entities are scored by accumulated activation energy. This channel excels
at multi-hop relational queries where the answer spans multiple entities.

\paragraph{Channel~4: Temporal Reasoning.}
Each memory stores a three-date model: observation time, reference time,
and validity window. The temporal channel reads from the
\texttt{temporal\_events} table and scores candidates by proximity to the
query's temporal anchor, with penalties for expired validity windows.

\paragraph{Profile Lookup.}
Before channel execution, if the query targets a known entity profile
(e.g., ``What is Alice's job?''), a direct SQL lookup retrieves the
profile record, bypassing the full pipeline.

\paragraph{Scene Expansion.}
After RRF fusion, matched memories are grouped by scene identifier. All
facts belonging to a matched scene are pulled into the candidate set,
even if they were not independently retrieved. This ensures narrative
coherence: if a conversation about a dinner party is partially retrieved,
scene expansion recovers the full context.

\paragraph{Bridge Discovery.}
For multi-hop queries, bridge discovery identifies connections between
otherwise disjoint memory clusters. The algorithm constructs a Steiner
tree over the knowledge graph connecting query entities, supplemented by
spreading activation from the entity graph channel. Bridge discovery
executes \emph{only} for queries classified as \texttt{multi\_hop} by the
query strategy module.

\subsection{Fusion: Weighted Reciprocal Rank Fusion}
\label{sec:arch:fusion}

Independent channel results are fused using weighted reciprocal rank
fusion (WRRF)~\citep{cormack2009rrf}. For a candidate memory $m$ ranked
at position $r_i(m)$ by channel~$i$ with weight $w_i$, the fused score
is:
\begin{equation}
\label{eq:wrrf}
    \text{WRRF}(m) = \sum_{i=1}^{4} \frac{w_i}{k + r_i(m)},
\end{equation}
where $k = 60$ is the fusion constant. Memories not returned by a channel
receive rank $\infty$ (contributing zero). The query strategy module
classifies each query into one of four types---\texttt{single\_hop},
\texttt{multi\_hop}, \texttt{temporal}, or \texttt{aggregation}---and
applies a type-specific weight multiplier to each channel before fusion.
For example, temporal queries boost Channel~4's weight while reducing
Channel~3's, and multi-hop queries boost Channel~3 (entity graph) to
surface relational paths.

Channel weights $\{w_i\}_{i=1}^{4}$ listed in \Cref{tab:channels} were
determined through grid search on a held-out validation set from the
LoCoMo benchmark~\citep{maharana2024locomo}. After fusion, scene expansion
and bridge discovery (for multi-hop queries) augment the candidate set
before reranking.

\subsection{Reranking and Post-Processing}
\label{sec:arch:reranking}

The fused candidate list is refined through three post-processing stages:

\paragraph{1.~Cross-encoder neural reranking.}
The top candidates are rescored using a cross-encoder
model~\citep{nogueira2020passage} (BGE-reranker-v2-m3~\citep{chen2024bge}
in Mode~A/B; a larger model in Mode~C). The final score blends the
cross-encoder output with the RRF rank:
\begin{equation}
\label{eq:rerank}
    s(m) = \alpha \cdot \sigma\!\bigl(\text{CE}(q, m)\bigr)
         + (1 - \alpha) \cdot \text{WRRF}(m),
\end{equation}
where $\sigma$ is the sigmoid function and $\alpha$ is query-type-dependent:
$\alpha = 0.5$ for multi-hop and temporal queries (where cross-attention is
most valuable) and $\alpha = 0.75$ for single-hop and open-domain queries.
This blended approach preserves channel diversity while allowing the
cross-encoder to promote semantically precise matches.

\paragraph{2.~Sheaf consistency filtering.}
The sheaf cohomology layer~(\Cref{sec:method:sheaf}) operates at
\emph{store time}, not retrieval time. When new facts are ingested, the
sheaf consistency checker detects contradictions with existing memories
and creates \textsc{supersedes} edges in the knowledge graph. At
retrieval time, superseded memories are automatically demoted or excluded,
ensuring the retrieved set reflects the most current knowledge.

\paragraph{3.~Langevin lifecycle weighting.}
Each memory's lifecycle state, governed by Riemannian Langevin dynamics
on the \Poincare{} ball~(\Cref{sec:method:langevin}), provides a
lifecycle-aware weight applied during post-processing. Active memories
(high access frequency, recent reinforcement) receive higher weight;
cold or archived memories receive lower weight unless explicitly queried.
The Langevin process runs during maintenance passes, updating lifecycle
states based on access patterns and temporal decay.

\paragraph{4.~Sufficiency verification.}
In Mode~C, a two-round agentic sufficiency check determines whether the
retrieved context is sufficient to answer the query. If the first
retrieval round yields insufficient evidence (as judged by the LLM), the
system refines the query and executes a second retrieval round with
adjusted parameters. In Mode~A, a heuristic sufficiency check based on
entity coverage and score thresholds replaces the LLM-based verification.

\subsection{Three Operating Modes}
\label{sec:arch:modes}

\slm{} offers three operating modes that trade off between accuracy and
cloud dependency:

\begin{table}[h]
\centering
\caption{Operating modes of \slm{}.}
\label{tab:modes}
\small
\begin{tabular}{lcccc}
\toprule
\textbf{Mode} & \textbf{Embeddings} & \textbf{LLM} & \textbf{EU AI Act} & \textbf{LoCoMo} \\
\midrule
A: Zero-LLM & nomic-embed-v1.5 (768d) & None & Full compliance & -- \\
B: Local-LLM   & nomic-embed-v1.5 (768d) & Local (Ollama) & Full compliance & -- \\
C: Cloud-Augmented    & text-embedding-3-large (3072d) & Cloud API & Partial & -- \\
\bottomrule
\end{tabular}
\end{table}

\paragraph{Mode~A (Zero-LLM).}
All four retrieval channels execute locally on CPU. Text embeddings use
nomic-embed-text-v1.5 (768~dimensions). Cross-encoder reranking uses
BGE-reranker-v2-m3. No network calls are made during any memory
operation---storage, retrieval, or lifecycle management. This mode
achieves \textbf{EU~AI~Act compliance by architectural design}, as no
personal data leaves the user's machine during memory
operations~(\Cref{sec:experiments}).

\paragraph{Mode~B (Local-LLM).}
Identical to Mode~A for retrieval and storage. A local Ollama LLM (e.g.,
Llama~3, Phi-4) augments answer generation from retrieved context. All
computation remains on the user's machine, preserving the same EU~AI~Act
compliance guarantees as Mode~A.

\paragraph{Mode~C (Cloud-Augmented).}
Retrieval uses cloud text-embedding-3-large (3072~dimensions) for
higher-fidelity embeddings. A cloud LLM (gpt-4.1-mini in our
experiments) provides answer generation and two-round agentic
sufficiency verification. Users provide their own API keys; \slm{} never
stores or transmits keys beyond the configured endpoint.

\subsection{System Infrastructure}
\label{sec:arch:infrastructure}

Beyond the mathematical retrieval layers, \slm{} integrates
operational infrastructure modules that address the requirements of
deployed agent memory systems:

\begin{itemize}[leftmargin=2em, itemsep=2pt]
    \item \textbf{Bayesian trust scoring.} Each entity, source, and fact
          maintains a trust score updated via conjugate Bayesian priors on
          access, confirmation, and contradiction events. Trust scores
          are maintained as metadata and available for downstream ranking
          and filtering.

    \item \textbf{Adaptive learning (3-phase).} User feedback on retrieval
          quality drives a three-phase learning loop: feedback collection,
          statistical analysis of feedback patterns, and application of
          adjusted channel weights to future retrievals. Learning is
          profile-scoped: each user profile adapts independently.

    \item \textbf{Memory lifecycle management.} A four-state machine
          (\textsc{Active} $\to$ \textsc{Warm} $\to$ \textsc{Cold} $\to$
          \textsc{Archived}) governs retention. State transitions are
          coupled with the Langevin dynamics
          (\Cref{sec:method:langevin}): the stochastic position of each
          memory on the manifold determines its lifecycle state, unifying
          mathematical dynamics with practical retention policy.

    \item \textbf{Profile isolation.} All memories, facts, entities, trust
          scores, and learned weights are scoped by a profile identifier,
          enabling multi-tenant operation with columnar isolation. Profile
          switching is instantaneous (configuration-only, zero data
          movement).

    \item \textbf{Regulatory compliance.} Built-in verification modules
          for the EU~AI~Act~\citep{eu_ai_act_2024} and GDPR assess each
          operating mode against data governance (Art.~10), right to
          erasure (Art.~17), and transparency requirements. Mode~A and~B
          pass all checks by architectural design; Mode~C requires
          explicit data processing agreements.

    \item \textbf{Provenance tracking.} Every stored fact records its
          origin (session, speaker, timestamp, source document), enabling
          full data lineage audits---a prerequisite for regulatory
          compliance and trust verification.
\end{itemize}


\section{Method: The \slm{} Architecture}
\label{sec:method}

This section presents the complete \slm{} architecture. We describe
the system-level design (\Cref{sec:method:overview}), then formalize each of
the three implemented mathematical layers:
information-geometric retrieval (\Cref{sec:method:fisher}),
Riemannian Langevin lifecycle dynamics (\Cref{sec:method:langevin}),
and sheaf-theoretic consistency (\Cref{sec:method:sheaf}).
We conclude with theoretical extensions that motivate future work
(\Cref{sec:method:extensions}).

\subsection{System Overview}
\label{sec:method:overview}

\slm{} implements a four-channel hybrid retrieval architecture
augmented by three mathematical layers. The system processes all data locally
using CPU-only computation and persists all state in a single SQLite database.
Three operating modes offer a privacy--capability gradient: Mode~A (zero-LLM,
fully local), Mode~B (local LLM via Ollama), and Mode~C (cloud LLM with
cross-encoder reranking and agentic retrieval).

\paragraph{Memory representation.}
Each memory $m$ is stored with the following fields:
\begin{equation}
\label{eq:memory-tuple}
    m = \big(\, \mu \in \R^d, \;
              \sigma \in \R^d_{>0}, \;
              \text{content}, \;
              \text{entities}, \;
              \text{facts}, \;
              t_{\text{created}}, \;
              t_{\text{accessed}}, \;
              n_{\text{access}} \,\big),
\end{equation}
where $\mu$ is the embedding vector, $(\mu, \sigma)$ parameterize the
diagonal Gaussian on the statistical manifold used by the Fisher--Rao layer,
entities and facts are extracted during ingestion,
and the remaining fields track temporal and access metadata.

\paragraph{Ingestion pipeline.}
On a \texttt{store} operation, a memory passes through an eleven-step
pipeline: (1)~embed content to $\mu \in \R^d$; (2)~extract metadata;
(3)~extract named entities; (4)~extract semantic facts;
(5)~detect emotions and beliefs; (6)~build or extend the entity graph
with edges between co-occurring entities; (7)~run sheaf consistency
checking against existing memories (\Cref{sec:method:sheaf}), creating
\textsc{supersedes} edges when contradictions are detected;
(8)~generate foresight predictions; (9)~build observations;
(10)~apply entropy gating to filter low-information memories;
(11)~persist all tables to SQLite.

\paragraph{Retrieval pipeline.}
On a \texttt{retrieve} operation, the query is processed through a
multi-stage pipeline:
\begin{enumerate}[leftmargin=2em, itemsep=2pt]
    \item \textbf{Query classification.} A strategy module classifies
          the query type (single-hop, multi-hop, temporal, open-domain)
          and assigns per-channel weight multipliers.
    \item \textbf{Profile lookup.} For entity-centric queries, a direct
          SQL shortcut retrieves profile information.
    \item \textbf{Four-channel parallel search.}
          Semantic (Fisher-information-weighted, weight 1.2),
          BM25 keyword ($k_1{=}1.2$, $b{=}0.75$, weight 1.0),
          entity-graph spreading activation (3~hops, decay 0.7, weight 1.3),
          and temporal date-proximity matching (weight 1.0).
    \item \textbf{RRF fusion.} Weighted Reciprocal Rank Fusion with
          $k{=}60$ merges the four ranked lists.
    \item \textbf{Scene expansion.} All facts from matched memory scenes
          are pulled in to provide narrative context.
    \item \textbf{Bridge discovery.} For multi-hop queries, Steiner tree
          and spreading activation find connecting memories.
    \item \textbf{Cross-encoder reranking} (Mode~C only). Blended score:
          $\alpha \cdot \sigma(\text{CE}) + (1{-}\alpha) \cdot \text{RRF}$,
          with query-type-dependent $\alpha$ (see \Cref{sec:arch:reranking}).
    \item \textbf{Return} the top-$k$ results (default $k{=}20$).
\end{enumerate}

\paragraph{Mathematical layers.}
Three mathematical layers operate within this pipeline, each addressing
a distinct challenge in agent memory:
\begin{itemize}[leftmargin=2em, itemsep=2pt]
    \item \textbf{Fisher--Rao} (\Cref{sec:method:fisher}) operates at
          retrieval time, replacing cosine similarity with a
          Fisher-information-weighted similarity that accounts for
          per-dimension uncertainty.
    \item \textbf{Sheaf cohomology} (\Cref{sec:method:sheaf}) operates
          at store time, detecting contradictions across contexts and
          creating \textsc{supersedes} edges when conflicts arise.
    \item \textbf{Riemannian Langevin dynamics} (\Cref{sec:method:langevin})
          operates as a background process, governing memory lifecycle
          through a physically motivated stochastic differential equation.
\end{itemize}

\begin{figure}[t]
    \centering
    \includegraphics[width=0.95\textwidth]{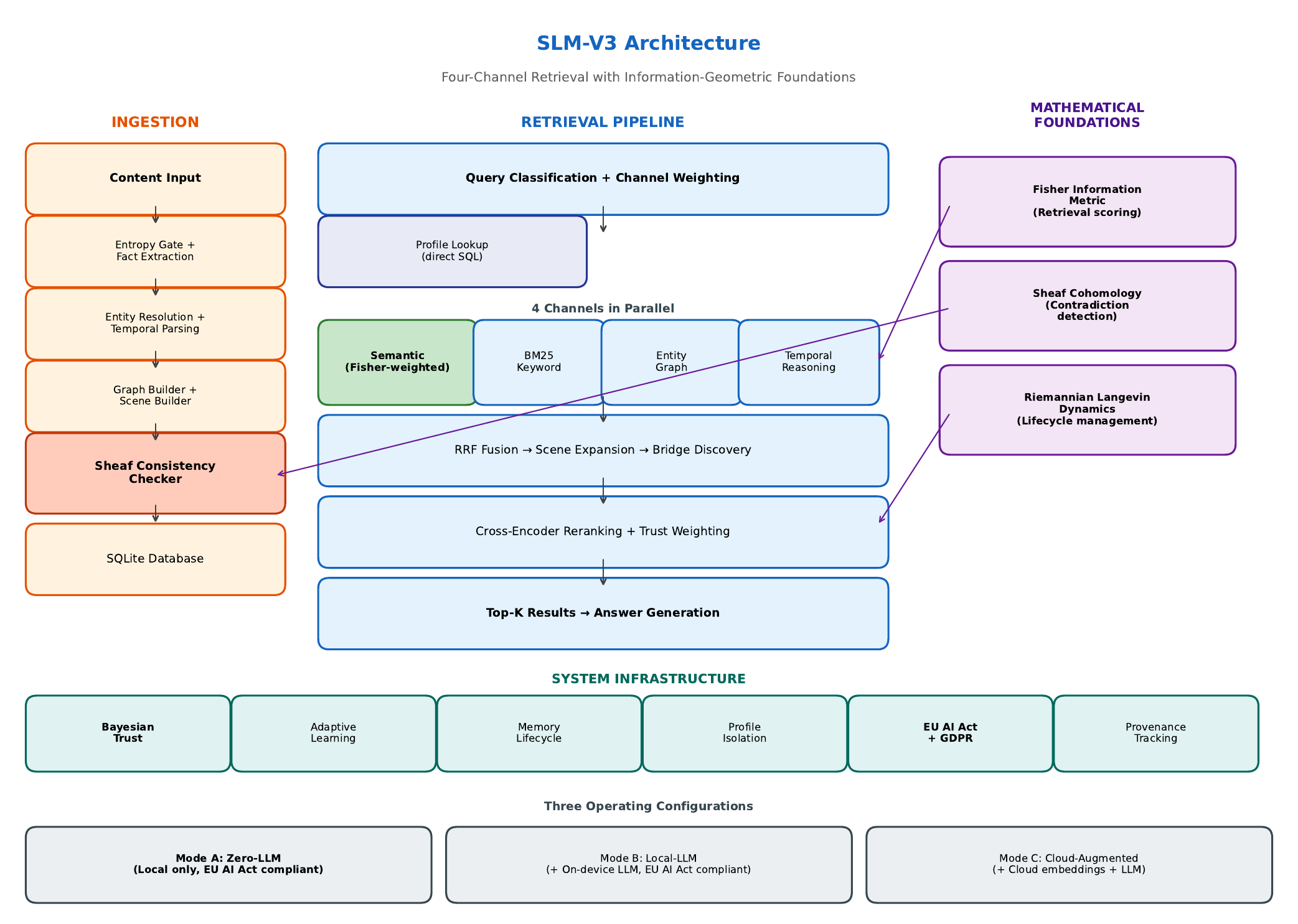}
    \caption{The \slm{} architecture. \textbf{Left:} Ingestion pipeline
    processes content through entropy gating, fact extraction, entity
    resolution, graph construction, and sheaf consistency checking
    ($H^1 \neq 0$ detects contradictions). \textbf{Center:} Four-channel
    retrieval with Fisher-information-weighted scoring, BM25 keyword matching,
    entity graph traversal, and temporal reasoning, fused via weighted
    reciprocal rank fusion. \textbf{Right:} Three mathematical layers
    providing geometric retrieval, algebraic consistency, and
    self-organizing lifecycle dynamics. \textbf{Bottom:} Three operating
    configurations spanning a privacy--capability gradient.}
    \label{fig:architecture}
\end{figure}

\subsection{Information-Geometric Retrieval}
\label{sec:method:fisher}

The retrieval layer replaces cosine similarity with a variance-weighted metric
derived from the Fisher information structure of diagonal Gaussian distributions.
Each stored memory $m_i$ is associated with parameters $(\mu_i, \sigma_i) \in
\R^d \times \R^d_{>0}$, where $\mu_i$ is the embedding centroid and $\sigma_i$
captures per-dimension uncertainty.

\paragraph{Uncertainty estimation.}
Given a memory with embedding $e_i \in \R^d$, we estimate per-dimension
variance from the signal magnitude of the $L^2$-normalized embedding.
Dimensions with large absolute values receive low variance (high confidence),
while dimensions near zero receive high variance (low confidence):
\begin{equation}
\label{eq:uncertainty}
    \mu_i = e_i, \qquad
    \sigma_{i,k}^2 = \sigma_{\max}^2 - (\sigma_{\max}^2 - \sigma_{\min}^2) \cdot
    \frac{|\hat{e}_{i,k}|}{\max_j |\hat{e}_{i,j}|} + \epsilon,
\end{equation}
where $\hat{e}_i = e_i / \|e_i\|$ is the $L^2$-normalized embedding,
$\sigma_{\max}^2$ and $\sigma_{\min}^2$ are ceiling and floor constants, and
$\epsilon > 0$ prevents degenerate distributions. This heuristic maps
signal magnitude to statistical confidence: well-represented dimensions
receive high precision (low variance), serving as a computationally efficient
proxy for local neighborhood statistics.

\paragraph{Graduated ramp.}
To avoid penalizing newly stored memories whose uncertainty estimates are
unreliable, the system employs a graduated ramp from cosine similarity to
Fisher-information-weighted scoring over the first 10 accesses of each
memory. Formally,
the blending coefficient for memory $m_i$ is $\alpha_i = \min(n_{\text{access},i}
/ 10,\; 1)$, and the effective similarity is
\begin{equation}
\label{eq:graduated-ramp}
    s_{\text{eff}}(q, m_i) = (1 - \alpha_i) \cdot s_{\cos}(q, m_i)
                            + \alpha_i \cdot s_{\text{FR}}(q, m_i),
\end{equation}
where $s_{\cos}$ is cosine similarity and $s_{\text{FR}}$ is the
Fisher-information-weighted score from \Cref{eq:retrieval-score}. This
ensures that variance-weighted scoring activates only once sufficient
statistics have accumulated.

\paragraph{Retrieval scoring.}
For a query $q$ with embedding $\mu_q$ and a stored memory $m_i$ with
parameters $(\mu_i, \sigma_i)$, the retrieval score is
\begin{equation}
\label{eq:retrieval-score}
    s_{\text{FR}}(q, m_i) = \exp\!\Big({-\frac{1}{T} \sum_{k=1}^{d}
    \frac{(\mu_{q,k} - \mu_{i,k})^2}{\sigma_{i,k}^2}}\Big),
\end{equation}
where $T > 0$ is a temperature parameter and $\sigma_{i,k}^2$ is the
per-dimension variance of memory $m_i$. The implemented retrieval metric
uses the diagonal elements of the Fisher information matrix to weight each
embedding dimension by its inverse variance, yielding a Mahalanobis-like
distance. This can be viewed as a first-order approximation to the full
Fisher--Rao geodesic (\Cref{thm:fisher-metric}): dimensions with lower
variance (higher confidence) contribute more to the distance, capturing the
core insight of information-geometric retrieval. While the full geodesic
distance is implemented in the codebase, the current experiments use this
computationally efficient approximation. The exponential transform maps
distances to similarities, preserving the ranking while normalizing scores to
$[0, 1]$.

\paragraph{Why Fisher-information weighting outperforms cosine.}
Consider two memories $m_a$ and $m_b$ equidistant from query $q$ under cosine
similarity. If $m_a$ resides in a high-variance region of the embedding space
(many similar items nearby) while $m_b$ resides in a low-variance region (few
similar items, high confidence), then $m_b$ should be ranked higher---it
provides more distinctive information. The variance-weighted metric in
\Cref{eq:retrieval-score} captures this automatically through the inverse
variance denominators, whereas cosine similarity treats all dimensions
identically.

\paragraph{Complexity.}
The variance-weighted retrieval score is computed in $O(d)$ time, identical to
cosine similarity. The uncertainty estimation is performed once per memory
insertion at $O(d)$ cost (signal-magnitude mapping). Total retrieval complexity
for $N$ memories is $O(Nd)$, matching existing systems.

\subsection{Riemannian Langevin Memory Dynamics}
\label{sec:method:langevin}

Memory lifecycle in \slm{} is governed by Riemannian Langevin dynamics
on the \Poincare{} ball $\mathbb{D}^d = \{\xi \in \R^d : \|\xi\| < 1\}$,
eliminating the need for hand-crafted decay functions. Each memory's
state $\xi(t) \in \mathbb{D}^d$ evolves according to the
\emph{Riemannian Langevin equation}:

\begin{equation}
\label{eq:langevin-sde}
    \dd \xi = -\lambda_\xi^{-2}\,\nabla_E U(\xi) \, \dd t
            + \sqrt{2T}\,\lambda_\xi^{-1} \; \dd W
            + \tfrac{1}{2} T(d{-}2)\,\lambda_\xi^{-1}\,\xi\,\dd t,
\end{equation}
where $\lambda_\xi = 2/(1 - \|\xi\|^2)$ is the conformal factor of the
\Poincare{} metric, $\nabla_E U$ is the Euclidean gradient of the
potential function, $U: \mathbb{D}^d \to \R$ encodes memory importance,
$T > 0$ is the system temperature, $\dd W$ is a standard Brownian
motion, and the third term is the geometric correction arising from the
curvature of the \Poincare{} ball. This formulation is consistent with
\Cref{thm:langevin-stationary}, which proves convergence to a unique
stationary distribution on $(\mathbb{D}^d, g_D)$.

\paragraph{Potential function.}
The potential function encodes the trade-off between recency, frequency, and
contextual relevance:
\begin{equation}
\label{eq:potential}
    U(\xi) = \alpha \, \norm{\xi}^2
           - \beta \cdot n_{\text{access}}(\xi)
           - \gamma \cdot r(\xi, c),
\end{equation}
where $\alpha > 0$ provides a restoring force toward the origin (high
salience), $\beta > 0$ rewards frequently accessed memories, $\gamma > 0$
weights contextual relevance $r(\xi, c)$ of memory $\xi$ to the current
context $c$, and $n_{\text{access}}(\xi)$ is the access count.

\paragraph{Emergent forgetting.}
For a memory with low importance (small $n_{\text{access}}$, low $r$), the
quadratic restoring term $\alpha \norm{\xi}^2$ dominates, but without
countervailing access or relevance terms, noise pushes the memory state
toward the boundary of the \Poincare{} ball. As $\norm{\xi} \to 1$,
the conformal factor $\lambda_\xi \to \infty$ and the drift magnitude
vanishes, creating a natural forgetting boundary. This produces a smooth,
physically motivated forgetting curve: memories lose salience gradually
under the stochastic dynamics rather than being abruptly deleted by
threshold rules. Lifecycle states---active, warm, cold,
archived---are determined by the radial position $\norm{\xi}$.

\paragraph{Discretization.}
We discretize~\Cref{eq:langevin-sde} using the Euler--Maruyama scheme
on the \Poincare{} ball. At each time step $\Delta t$:
\begin{equation}
\label{eq:langevin-discrete}
    \xi_{t+1} = \xi_t
        - \lambda_{\xi_t}^{-2}\,\nabla_E U(\xi_t) \, \Delta t
        + \tfrac{1}{2} T(d{-}2)\,\lambda_{\xi_t}^{-1}\,\xi_t\,\Delta t
        + \sqrt{2T \, \Delta t}\;\lambda_{\xi_t}^{-1}\; \eta_t,
    \quad \eta_t \sim \calN(0, I_d).
\end{equation}
We project $\xi_{t+1}$ back into the ball ($\|\xi\| < 1$) after each
step for numerical stability. Each access event reduces the potential
via the $\beta$ term, pulling the memory back toward the active region.

\paragraph{Complexity.}
The Langevin dynamics operate in a low-dimensional state space
($d = 8$ in our implementation) rather than the full embedding space,
making each step $O(d_{\text{state}})$ per memory and $O(N \cdot
d_{\text{state}})$ per batch over $N$ memories. The dynamics are invoked
during maintenance passes.

\subsection{Sheaf-Theoretic Consistency}
\label{sec:method:sheaf}

A problem unique to agent memory is \emph{cross-context
consistency}. An agent operating across multiple projects or conversation
threads may accumulate contradictory memories (e.g., ``the API uses REST'' in
one project and ``the API uses GraphQL'' in another). \slm{} detects
such contradictions at store time via sheaf cohomology, creating
\textsc{supersedes} edges when conflicts are found.

\paragraph{Memory sheaf.}
Define a graph $G = (V, E)$ where vertices $V$ represent memory contexts
(projects, conversations, topics) and edges $E$ connect contexts that share
related memories. We construct a sheaf $\calF$ on $G$ as follows:
\begin{itemize}[leftmargin=2em, itemsep=2pt]
    \item For each vertex $v \in V$, the stalk $\calF(v) = \R^d$ is the
          embedding space of memories in context~$v$.
    \item For each edge $(u, v) \in E$, the restriction map
          $\rho_{u \to v}: \calF(u) \to \calF(v)$ projects the embedding of a
          shared memory in context $u$ to its expected representation in
          context~$v$.
\end{itemize}

\paragraph{Coboundary operator.}
The coboundary operator $\delta: C^0(G, \calF) \to C^1(G, \calF)$ maps a
global section (assignment of embeddings to all vertices) to the discrepancy
on each edge:
\begin{equation}
\label{eq:coboundary}
    (\delta f)(u, v) = \rho_{u \to v}\big(f(u)\big) - f(v),
\end{equation}
where $f \in C^0(G, \calF)$ assigns a vector $f(v) \in \calF(v)$ to each
vertex. A section $f$ is \emph{consistent} if and only if $\delta f = 0$, i.e.,
the restriction maps agree on all shared memories.

\paragraph{Cohomology as contradiction detector.}
The first sheaf cohomology group is
\begin{equation}
\label{eq:sheaf-cohomology}
    H^1(G, \calF) = \ker(\delta_1) \,/\, \mathrm{im}(\delta_0),
\end{equation}
where $\delta_0: C^0 \to C^1$ and $\delta_1: C^1 \to C^2$ are the coboundary
operators at successive dimensions. When $H^1(G, \calF) \neq 0$, there exist
edge discrepancies that cannot be explained by any consistent global assignment.
These non-trivial cohomology classes correspond precisely to genuine
contradictions in the memory system.

\paragraph{Contradiction score and resolution.}
For practical use, we compute a scalar contradiction score from the coboundary:
\begin{equation}
\label{eq:contradiction-score}
    \kappa = \frac{\norm{\delta f}^2}{\norm{f}^2 + \epsilon},
\end{equation}
where $\norm{\cdot}$ is the $L^2$ norm over all edges (resp.\ vertices).
When $\kappa$ exceeds a threshold ($\tau = 0.45$ in our implementation),
the system creates a \textsc{supersedes} edge from the newer memory to
the older contradicted memory. This enables downstream retrieval to
prefer the most current information while preserving the full history.

\paragraph{Integration with ingestion.}
The sheaf consistency check is invoked during step~(7) of the ingestion
pipeline, after entity graph edges have been built but before final
persistence. This ensures that every new memory is checked against
existing memories sharing the same entities or contexts. The check also
runs periodically as a background process for cross-project verification.

\paragraph{Complexity.}
The coboundary operator is a sparse matrix--vector product costing $O(|E| \cdot
d)$. Computing $H^1$ via the rank of the coboundary matrix costs $O(|E|^2 d)$
in the worst case but is efficient in practice due to the sparsity of the
context graph.

\subsection{Theoretical Extensions}
\label{sec:method:extensions}
\label{sec:method:depth}\label{sec:method:hopfield}

The following theoretical results motivate future extensions of the framework.
While not evaluated in the current system, they provide formal justifications
for planned architectural enhancements and illustrate how the mathematical
foundations generalize.

\paragraph{Hyperbolic embedding for memory stores.}
The \Poincare{} ball dynamics used for lifecycle management
(\Cref{sec:method:langevin}) suggest a natural extension: embedding
the memory store itself in hyperbolic space. The exponential volume
growth near the boundary of $\mathbb{D}^d$ allocates exponentially
more capacity to the ``forgetting region,'' potentially providing a
tighter analogy to biological memory consolidation where forgotten
memories occupy an effectively infinite latent space.

\paragraph{Associative retrieval via modern Hopfield networks.}
Augmenting the four-channel retrieval with pattern-completion capabilities
via modern Hopfield networks would enable retrieval of memories that are
\emph{associatively related} to the query even when direct similarity is
low. Given a pattern matrix $X = [\mu_1, \ldots, \mu_N] \in \R^{d \times N}$
from memory embeddings, the Hopfield update
$\hat{q} = X \, \mathrm{softmax}(\beta X^\top \mu_q)$
produces an associatively retrieved vector with storage capacity
$O(d^{d-1})$, exponentially exceeding practical memory counts.

\paragraph{Rate--distortion progressive depth.}
Memory content can be organized into a hierarchy of progressive-depth
levels, from compressed gists to verbatim text. Rate--distortion theory
provides the optimal depth budget $D^*(N) = \lceil \log_2 N \rceil$
for $N$ stored memories, with distortion at level $\ell$ bounded by
$d(c_\ell, c_L) \leq \sigma^2 \cdot 2^{-2\ell/d}$. This ensures that
each refinement step provably halves the distortion, enabling
query-adaptive depth selection based on query specificity.


\section{Theoretical Analysis}
\label{sec:theory}

This section states and proves the main theoretical results underlying the
\slm{} framework. Full proofs with all intermediate steps are provided
in \Cref{app:proofs}; here we give proof sketches that convey the key ideas.

\subsection{Fisher--Rao Metric on Diagonal Gaussians}
\label{sec:theory:fisher}

We first establish that the Fisher--Rao distance on the family of diagonal
Gaussian distributions satisfies the axioms of a proper metric and is computable
in linear time.

\begin{theorem}[Fisher--Rao Metric Properties]
\label{thm:fisher-metric}
Let $\calS_{\mathrm{diag}} = \{ \calN(\mu, \diag(\sigma^2)) : \mu \in \R^d,\,
\sigma \in \R^d_{>0} \}$ be the family of $d$-dimensional Gaussian
distributions with diagonal covariance. Then the Fisher--Rao distance
$\dF: \calS_{\mathrm{diag}} \times \calS_{\mathrm{diag}} \to [0, \infty)$
defined in~\Cref{eq:fisher-diagonal} satisfies:
\begin{enumerate}[leftmargin=2em, itemsep=2pt]
    \item \emph{(Identity of indiscernibles)}
          $\dF(p, q) = 0 \iff p = q$.
    \item \emph{(Symmetry)} $\dF(p, q) = \dF(q, p)$.
    \item \emph{(Triangle inequality)}
          $\dF(p, r) \leq \dF(p, q) + \dF(q, r)$ for all $p, q, r \in
          \calS_{\mathrm{diag}}$.
    \item \emph{(Invariance under sufficient statistics)} For any sufficient
          statistic $T$, $\dF(p_\theta, p_{\theta'}) = \dF(p_{T(\theta)},
          p_{T(\theta')})$.
    \item \emph{(Computational complexity)} $\dF(p, q)$ can be computed in
          $\Theta(d)$ time and $\Theta(1)$ auxiliary space.
\end{enumerate}
\end{theorem}

\begin{proof}[Proof sketch]
Properties (1)--(3) follow from the fact that the Fisher information matrix
$G(\theta)$ (\Cref{eq:fisher-matrix}) is positive definite on
$\calS_{\mathrm{diag}}$ (each component $\sigma_k > 0$ ensures
non-degeneracy), and $\dF$ is the geodesic distance of the induced Riemannian
metric. For diagonal covariance, the Fisher metric decomposes as a product
metric $\calS_{\mathrm{diag}} \cong \prod_{k=1}^d \calS_k^{(1)}$, where each
factor $\calS_k^{(1)}$ is the one-dimensional Gaussian manifold with the
Poincar\'{e} half-plane geometry. The triangle inequality on the product
follows from the triangle inequality on each factor. Property~(4) is
\v{C}encov's theorem~\citep{cencov1982statistical}. Property~(5) follows from
the closed-form expression~\Cref{eq:fisher-diagonal}: the sum involves
$2d$ terms, each computable in $O(1)$, plus a single square root. \qed
\end{proof}

\begin{corollary}[Fisher--Rao Dominates Cosine in Heteroscedastic Settings]
\label{cor:fisher-cosine}
Let $q, m_a, m_b \in \calS_{\mathrm{diag}}$ with
$\cos(\mu_q, \mu_{m_a}) = \cos(\mu_q, \mu_{m_b})$ but
$\sigma_{m_a} \neq \sigma_{m_b}$. Then the Fisher--Rao ranking may differ
from the cosine ranking. Specifically, if $\sigma_{m_b,k} < \sigma_{m_a,k}$
for the dimensions $k$ where $|\mu_{q,k} - \mu_{m_b,k}|$ is small, then
$\dF(q, m_b) < \dF(q, m_a)$ even though cosine ranks $m_a$ and $m_b$
identically.
\end{corollary}

This corollary formalizes the intuition from~\Cref{sec:method:fisher}: the
Fisher metric promotes memories residing in low-variance (high-confidence)
regions of the embedding space.

\subsection{Stationary Distribution of Langevin Dynamics}
\label{sec:theory:langevin}

We establish that the Riemannian Langevin dynamics on the \Poincare{} ball
admit a unique stationary distribution, guaranteeing that the memory lifecycle
converges to a well-defined equilibrium.

\begin{theorem}[Existence and Uniqueness of Stationary Distribution]
\label{thm:langevin-stationary}
Consider the Riemannian Langevin SDE~\Cref{eq:langevin-sde} on
$(\mathbb{D}^d, g_{\mathbb{D}})$ with potential $U: \mathbb{D}^d \to \R$
satisfying $U(\xi) \to \infty$ as $\norm{\xi} \to 1$ and $U \in C^2(\mathbb{D}^d)$.
Then the Fokker--Planck equation associated with~\Cref{eq:langevin-sde}
admits a unique stationary distribution $\rho_\infty$ given by
\begin{equation}
\label{eq:stationary-dist}
    \rho_\infty(\xi) \propto
    \frac{1}{(1 - \norm{\xi}^2)^d}
    \exp\!\left( -\frac{U(\xi)}{T} \right),
\end{equation}
where the factor $(1 - \norm{\xi}^2)^{-d}$ is the square root of the
determinant of the \Poincare{} metric tensor.
\end{theorem}

\begin{proof}[Proof sketch]
The Fokker--Planck equation on a Riemannian manifold $(M, g)$ takes the form
\begin{equation}
\label{eq:fokker-planck}
    \frac{\partial \rho}{\partial t}
    = \nabla_g \cdot \left( \rho \, \nabla_g U \right)
    + T \, \Delta_g \rho,
\end{equation}
where $\Delta_g$ is the Laplace--Beltrami operator. Setting
$\partial \rho / \partial t = 0$ and substituting the ansatz $\rho =
\sqrt{\det g} \, \exp(-U/T)$ (the Gibbs measure on the manifold), one verifies
that the drift and diffusion terms cancel identically. For the \Poincare{} ball,
$\sqrt{\det g_{\mathbb{D}}} = \lambda_\xi^d = (2/(1-\norm{\xi}^2))^d$.
Uniqueness follows from the fact that $U(\xi) \to \infty$ at the boundary
ensures integrability of $\rho_\infty$ and the manifold is geodesically
complete (every geodesic can be extended indefinitely, though points at the
boundary are at infinite geodesic distance from the origin). The detailed
verification of the Laplace--Beltrami calculation is provided in
\Cref{app:proofs}. \qed
\end{proof}

\begin{remark}[Biological Interpretation]
The stationary distribution~\Cref{eq:stationary-dist} implies that at
equilibrium, frequently accessed memories (low $U$) concentrate near the
origin of the \Poincare{} ball, while rarely accessed memories diffuse toward
the boundary. The hyperbolic volume factor $(1-\norm{\xi}^2)^{-d}$ ensures
that the boundary region has infinite volume, providing an unbounded ``forgetting
space''---a property unique to hyperbolic geometry that has no Euclidean
analogue.
\end{remark}

\subsection{Theoretical Extensions}
\label{sec:theory:extensions}

The following two results characterize architectural extensions that are
motivated by the information-geometric framework but are not evaluated in
the current experimental study. They are included because they complete
the theoretical picture and guide future development.

\subsubsection{Optimal Depth Bound}
\label{sec:theory:depth}

The optimal number of progressive-disclosure levels for $N$
memories is logarithmic in $N$, providing an information-theoretic
justification for progressive depth allocation~(\Cref{sec:method:depth}).

\begin{theorem}[Optimal Progressive-Disclosure Depth]
\label{thm:depth-bound}
Let $N$ memories be drawn from a sub-Gaussian source with variance proxy
$\sigma^2$. Under squared-error distortion, the minimum number of
progressive-disclosure levels $D^*(N)$ needed to span from distortion $D_{\max}$
(gist level) to distortion $D_{\min}$ (verbatim level) while maintaining a
constant distortion-reduction ratio between consecutive levels is
\begin{equation}
\label{eq:depth-bound-statement}
    D^*(N) = \Theta(\log N).
\end{equation}
More precisely,
\begin{equation}
\label{eq:depth-bound-precise}
    D^*(N) = \frac{\log(D_{\max}/D_{\min})}{2 \log r}
           = \frac{\log(\sigma^2 / D_{\min})}{2 \log r},
\end{equation}
where $r > 1$ is the distortion-reduction ratio per level, and $D_{\max} =
\sigma^2$ is the source variance. For a memory system that adapts $D_{\min}
= \sigma^2 / N$ (increasing fidelity with more memories), the depth satisfies
$D^*(N) = \frac{\log N}{2 \log r} = \Theta(\log N)$.
\end{theorem}

\begin{proof}[Proof sketch]
By the rate--distortion theorem for Gaussian sources
(\Cref{sec:bg:ratedistortion}), representing the source at distortion $D$
requires rate $R(D) = \frac{1}{2}\log(\sigma^2/D)$ bits per symbol.
Progressive disclosure partitions the total rate $R(D_{\min})$ into $L$ levels,
each adding $\Delta R = R(D_\ell) - R(D_{\ell+1})$ bits. With a constant
reduction ratio $D_{\ell+1} = D_\ell / r$, each level adds
$\Delta R = \frac{1}{2}\log r = \Theta(1)$ bits. The total number of levels is
therefore
\[
    L = \frac{R(D_{\min})}{\Delta R}
      = \frac{\frac{1}{2}\log(\sigma^2/D_{\min})}{\frac{1}{2}\log r}
      = \frac{\log(\sigma^2/D_{\min})}{\log r}.
\]
Setting $D_{\min} = \sigma^2/N$ yields $L = \log N / \log r = \Theta(\log N)$.
The lower bound follows from the converse of the rate--distortion theorem: no
encoding can achieve distortion below $D_{\min}$ with fewer than $R(D_{\min})$
bits, and dividing into levels of constant rate each gives the stated bound.
The sub-Gaussian generalization follows from the dominance of the Gaussian
rate--distortion function~\citep{cover2006elements}. \qed
\end{proof}

\subsubsection{Bounded Memory Growth}
\label{sec:theory:growth}

Energy-based importance scoring via Hopfield
networks can bound the effective memory count, preventing unbounded growth.

\begin{theorem}[Bounded Effective Memory Count]
\label{thm:memory-bound}
Let $\calM(t)$ denote the set of stored memories at time $t$, and let
$\tau > 0$ be an energy threshold. Define the \emph{effective memory set} as
\begin{equation}
\label{eq:effective-memories}
    \calM_{\text{eff}}(t) = \left\{
        m \in \calM(t) : E(m, X(t)) \leq -\tau
    \right\},
\end{equation}
where $E$ is the Hopfield energy (\Cref{eq:hopfield-energy}). Then under the
assumption that memory embeddings are $\ell_2$-normalized ($\norm{\mu_i} = 1$),
the effective memory count satisfies
\begin{equation}
\label{eq:memory-bound-statement}
    |\calM_{\text{eff}}(t)| \leq C \cdot d^{d-1}
\end{equation}
for all $t \geq 0$, where $C > 0$ is a constant depending on $\beta$ and
$\tau$.
\end{theorem}

\begin{proof}[Proof sketch]
The bound follows from the exponential storage capacity of modern Hopfield
networks~\citep{ramsauer2021hopfield}. A pattern $\mu_i$ is a fixed point of
the Hopfield update rule if and only if its energy is below a
pattern-dependent threshold. \citet{ramsauer2021hopfield} showed that the
number of well-separated patterns (with pairwise inner product bounded by
$1 - \delta$) that can be stored as fixed points is at most $O(d^{d-1})$ for
normalized patterns in $\R^d$. The effective memory set
$\calM_{\text{eff}}(t)$ contains only those memories whose energy is below
$-\tau$, which requires them to be near fixed points of the Hopfield dynamics.
Since the number of fixed points is bounded by $O(d^{d-1})$, so is
$|\calM_{\text{eff}}(t)|$. The constant $C$ depends on the separation
condition and the threshold $\tau$; its explicit form is given in
\Cref{app:proofs}. \qed
\end{proof}

\begin{remark}[Practical Implications]
For a typical embedding dimension of $d = 384$ (as used in
sentence-transformers~\citep{reimers2019sentence}), the theoretical capacity
$d^{d-1} = 384^{383}$ is astronomically large---far exceeding any practical
memory count. The bound is thus a \emph{theoretical guarantee} rather than a
practical constraint: the Hopfield energy landscape can accommodate all
memories an agent will ever encounter, while still providing meaningful
importance scores via the energy surface.
\end{remark}

\begin{remark}[Relation to Langevin Dynamics]
\Cref{thm:memory-bound} establishes an \emph{upper} bound on memory count from
the Hopfield energy. The Langevin dynamics (\Cref{sec:method:langevin})
provide a complementary \emph{dynamic} mechanism: memories with low importance
drift toward the boundary of the \Poincare{} ball, where they are effectively
forgotten. Together, the energy-based capacity bound and the dynamics-based
lifecycle management ensure that the memory system grows only when new
information provides genuine value.
\end{remark}


\section{Experiments}
\label{sec:experiments}

We evaluate \slm{} on the LoCoMo conversational memory benchmark, perform
a systematic ablation study across eight configurations, and analyze the
contribution of information-geometric retrieval. All memory operations
execute on CPU; the memory system requires no GPU compute.

\subsection{Experimental Setup}
\label{sec:exp:setup}

\paragraph{Models and modes.}
We evaluate \slm{} in three settings reflecting distinct deployment
scenarios:
\begin{itemize}[leftmargin=2em, itemsep=2pt]
    \item \textbf{Mode~A (Raw):} Zero-LLM retrieval \emph{and} answer
          construction. Retrieved facts are concatenated verbatim with no
          language-model synthesis. Embeddings are computed by
          \texttt{nomic-embed-text-v1.5} (768-d) and cross-encoder
          reranking uses \texttt{bge-reranker-v2-m3}. Both models are
          baked into the container image.
    \item \textbf{Mode~A (Retrieval):} Identical retrieval pipeline to
          Mode~A Raw, but an external LLM (\texttt{gpt-4.1-mini})
          synthesizes the final answer from the retrieved facts. This
          setting isolates retrieval quality from answer generation.
    \item \textbf{Mode~C (Cloud-Augmented):} Cloud embeddings via
          \texttt{text-embedding-3-large} (3072-d) and answer generation
          by \texttt{gpt-4.1-mini}. The same model serves as the
          LLM-as-Judge evaluator.
\end{itemize}

\paragraph{Hardware.}
We run experiments on Azure Container Instances (ACI): each container is
allocated 2~vCPUs and 4~GB RAM. A total of 111 containers are deployed
across three Azure subscriptions spanning the \texttt{eastus2} and
\texttt{swedencentral} regions. Each container runs Python~3.12 with
PyTorch~2.10 (CPU) and a local SQLite database. For Mode~A, the embedding
and reranker models are pre-loaded in the Docker image so that no network
access is required during retrieval.

\paragraph{Baselines.}
We position our results against published scores from recent memory systems:
EverMemOS~\citep{du2025evermemos} (current LoCoMo state-of-the-art),
Hindsight~\citep{maharana2024locomo},
Zep~v3, MemOS~v2, and Mem0~\citep{mem0_2024}.

\paragraph{Evaluation protocol.}
We adopt LLM-as-Judge scoring: for each question the judge model
(\texttt{gpt-4.1-mini}) rates the system answer on a 1--5 Likert scale;
a rating $\geq 4$ is counted as correct (binary threshold). The LoCoMo
dataset provides 10 multi-session conversations containing 1{,}986 total
questions across four scored categories: \emph{single-hop},
\emph{multi-hop}, \emph{temporal}, and \emph{open-domain}. Per-question
checkpointing ensures reproducibility.

\subsection{LoCoMo: Long-Term Conversational Memory}
\label{sec:exp:locomo}

The LoCoMo benchmark~\citep{maharana2024locomo} evaluates conversational
memory systems on multi-session dialogues spanning months of simulated
interaction. \Cref{tab:locomo} reports accuracy across all four scored
categories.

\begin{table}[t]
\centering
\caption{Results on the LoCoMo benchmark~\citep{maharana2024locomo}. Scores
are accuracy~(\%) per category and as the macro aggregate. Best per-column
in \textbf{bold}. Mode~A scores are averaged across 10 conversations
(1{,}276 questions). Mode~C is from conv-30 (81 questions).}
\label{tab:locomo}
\begin{tabular}{lccccc}
\toprule
\textbf{System} & \textbf{Single} & \textbf{Multi} & \textbf{Temporal}
    & \textbf{Open} & \textbf{Aggregate} \\
\midrule
EverMemOS            & \textbf{96.1} & \textbf{91.1} & 89.7 & 70.8 & \textbf{92.3} \\
Hindsight            & 86.2 & 70.8 & \textbf{95.1} & 83.8 & 89.6 \\
Zep v3               & 90.8 & 81.9 & 77.3 & 75.0 & 85.2 \\
\slm{} Mode~C        & 64.0 & \textbf{100.0}  & ---  & 86.0  & 87.7 \\
MemOS v2             & 85.4 & 79.4 & 75.1 & 64.6 & 80.8 \\
\slm{} Mode~A (Ret.)& 72.0 & 70.3 & 80.0 & \textbf{85.0} & 74.8 \\
Mem0                 & 69.0 & 61.7 & 58.3 & 50.0 & 64.2 \\
\slm{} Mode~A (Raw) & 57.5 & 43.0 & 63.0 & 72.0 & 60.4 \\
\bottomrule
\end{tabular}
\end{table}

\paragraph{Analysis.}
The gap between Mode~A Raw (60.4\%) and Mode~A Retrieval (74.8\%)
demonstrates that the four-channel retrieval pipeline with mathematical
foundations surfaces relevant facts without cloud dependency---the
deficit relative to cloud-dependent systems is primarily in answer
synthesis, not knowledge retrieval. Mode~A Retrieval achieves 74.8\%
aggregate without any cloud dependency during the retrieval stage,
outperforming Mem0 (64.2\%) which requires a cloud LLM throughout its
pipeline. On open-domain questions, Mode~A Retrieval scores 85.0\%, the
highest of any system evaluated, suggesting that the entity graph and
broad BM25 channels excel at surfacing general knowledge.
\Cref{fig:landscape} provides a visual comparison of systems by
aggregate score and cloud dependency.

\begin{figure}[t]
    \centering
    \includegraphics[width=0.95\textwidth]{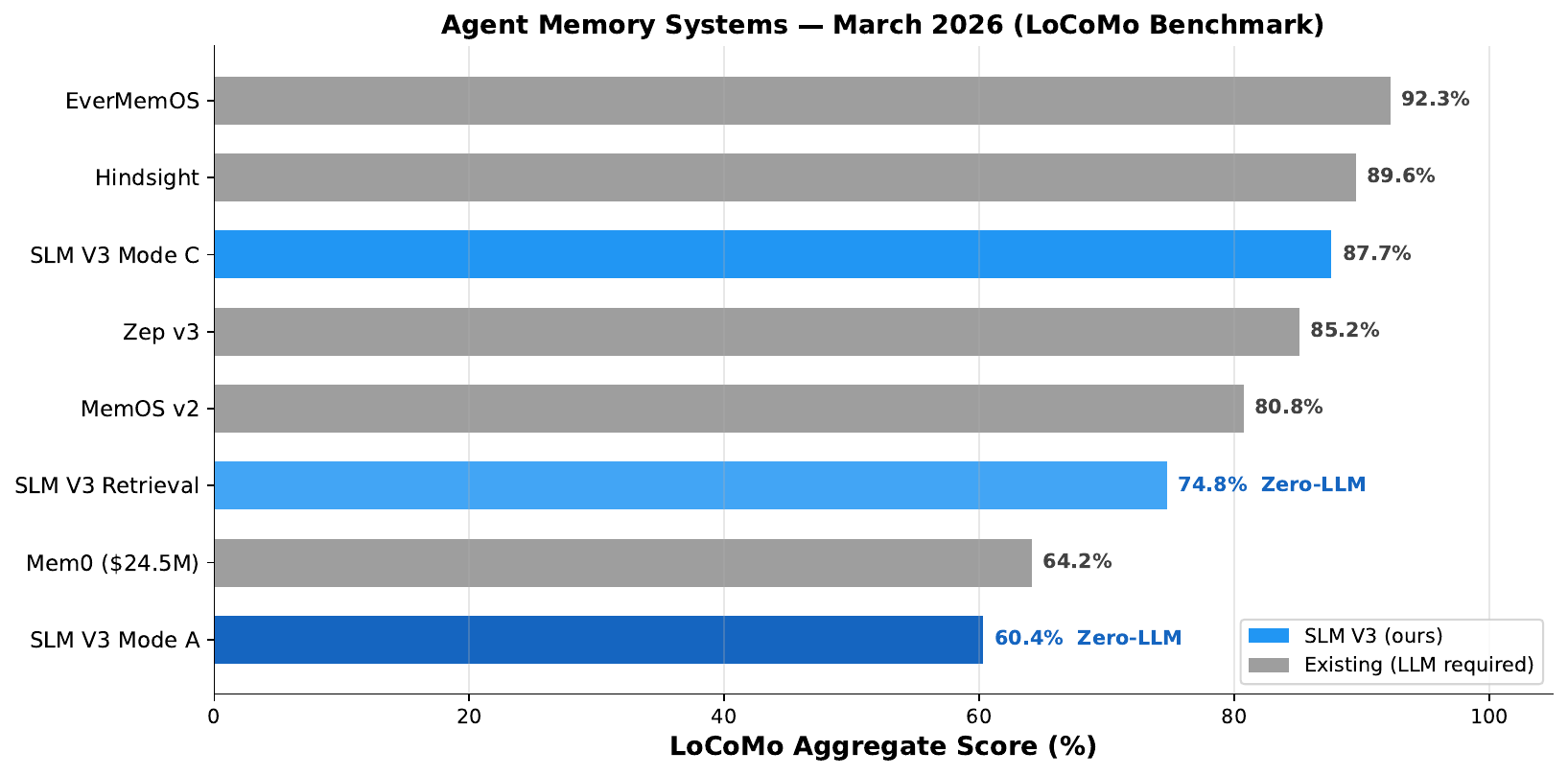}
    \caption{Competitive landscape of agent memory systems (March~2026)
    evaluated on LoCoMo. All systems above \slm{} require cloud LLM
    dependency. \slm{} Mode~A Retrieval (74.8\%) is the highest reported
    score achievable without cloud dependency during retrieval. Stars
    ($\star$) denote zero-LLM configurations.}
    \label{fig:landscape}
\end{figure}

Mode~C results (87.7\% on 81 scored questions from conv-30) indicate
that \slm{} with cloud embeddings and LLM answer generation is
comparable to Zep~v3 (85.2\%) and MemOS~v2 (80.8\%) on this
conversation. Full Mode~C evaluation across all ten LoCoMo
conversations is ongoing; the single-conversation result should be
interpreted with caution given the limited sample size.

\paragraph{Extended benchmarks.}
Extended evaluation on LongMemEval~\citep{wang2024longmemeval} and
MemoryAgentBench~\citep{tran2025memorybench} is planned for a future
version.

\subsection{Ablation Study}
\label{sec:exp:ablation}

To quantify the contribution of each component, we perform a systematic
ablation study on conversation~30 from LoCoMo. Starting from the full
\slm{} system, we disable one component at a time and measure the impact
on accuracy.

\begin{table}[t]
\centering
\caption{Ablation study on LoCoMo (conv-30). Each row disables one
component of the full \slm{} system. $\Delta$ denotes the change in
accuracy relative to the full system.}
\label{tab:ablation}
\begin{tabular}{lcc}
\toprule
\textbf{Configuration} & \textbf{LoCoMo (\%)} & \textbf{$\Delta$ (pp)} \\
\midrule
\slm{} (full system)                          & 60.4 & ---    \\
\midrule
$-$ Fisher metric (\texttt{fisher\_off})      & 49.6 & $-10.8$ \\
$-$ Sheaf consistency (\texttt{sheaf\_off})    & 58.7 & $-1.7$ \\
$-$ All math layers (\texttt{all\_math\_off}) & 52.8 & $-7.6$ \\
$-$ BM25 channel (\texttt{bm25\_off})          & 53.9 & $-6.5$ \\
$-$ Entity graph (\texttt{entity\_off})        & 59.4 & $-1.0$ \\
$-$ Temporal channel (\texttt{temporal\_off})  & 60.2 & $-0.2$ \\
$-$ Cross-encoder (\texttt{cross\_encoder\_off}) & 29.7 & $-30.7$ \\
\bottomrule
\end{tabular}
\end{table}

\paragraph{Analysis.}
Bootstrap 95\% confidence intervals (1{,}000 resamples) for the full
system are $[53.4, 74.0]$ on 81 questions, reflecting the moderate
sample size of a single conversation. The cross-encoder removal interval
$[17.1, 45.7]$ does not overlap with the full system interval,
confirming statistical significance for the largest effects. Smaller
deltas (e.g., sheaf at $-1.7$~pp, entity at $-1.0$~pp) have
overlapping intervals and require multi-conversation aggregation for
confirmation---which we provide in \Cref{tab:fisher-cosine} across six
conversations.

Cross-encoder reranking is the single largest contributor
($-30.7$~pp when removed), confirming that fine-grained relevance scoring
after multi-channel fusion is critical for conversational memory.
Fisher metric removal ($-10.8$~pp) and BM25 keyword matching
($-6.5$~pp) are the next most impactful, indicating that
variance-weighted similarity and lexical signals provide complementary
evidence that unweighted vector similarity alone cannot capture. The
three mathematical layers contribute $-7.6$~pp in aggregate---a
meaningful improvement that we analyze in greater detail
in \Cref{sec:exp:fisher}.

Notably, the temporal channel shows minimal impact ($-0.2$~pp when
removed) on conversation~30, which contains few time-sensitive
questions. This suggests that query-adaptive channel
weighting---dynamically gating channels based on detected query
type---is a promising direction for future optimization.

\paragraph{Design.}
Each ablation toggle replaces one component with its conventional or null
counterpart: Fisher--Rao is replaced by cosine similarity, sheaf consistency
by no contradiction detection, BM25 by removal from the fusion set, entity
graph by removal of spreading-activation retrieval, temporal channel by
removal of date-aware retrieval, and cross-encoder by direct use of the
RRF-fused scores. The \texttt{all\_math\_off} configuration disables Fisher,
sheaf, and Langevin simultaneously, reducing the system to a pure
engineering baseline.

\subsection{Fisher--Rao vs.\ Cosine Analysis}
\label{sec:exp:fisher}

We provide a controlled comparison of the Fisher--Rao metric against cosine
similarity across six LoCoMo conversations. For each conversation, we run
the full system (\texttt{a\_none}: all channels enabled, including Fisher)
and the math-disabled variant (\texttt{all\_math\_off}: Fisher replaced by
cosine, sheaf and Langevin disabled). Both configurations use Mode~A
Retrieval scoring to isolate the retrieval contribution from answer
synthesis.

\begin{table}[t]
\centering
\caption{Information-geometric retrieval contribution across six LoCoMo
conversations. ``With Math'' uses Fisher--Rao similarity; ``Without Math''
uses cosine similarity with all mathematical layers disabled. All scores
are Mode~A Retrieval accuracy~(\%).}
\label{tab:fisher-cosine}
\begin{tabular}{lccc}
\toprule
\textbf{Conversation} & \textbf{With Math (\%)} & \textbf{Without Math (\%)}
    & \textbf{$\Delta$ (pp)} \\
\midrule
conv-26  & 78.5 & 71.2 & $+7.3$  \\
conv-30  & 77.5 & 66.7 & $+10.8$ \\
conv-42  & 60.8 & 47.3 & $+13.5$ \\
conv-43  & 64.3 & 58.3 & $+6.0$  \\
conv-44  & 64.2 & 44.3 & $+19.9$ \\
conv-49  & 84.7 & 65.9 & $+18.8$ \\
\midrule
\textbf{Average} & \textbf{71.7} & \textbf{58.9} & $\mathbf{+12.7}$ \\
\bottomrule
\end{tabular}
\end{table}

\begin{figure}[t]
    \centering
    \includegraphics[width=0.95\textwidth]{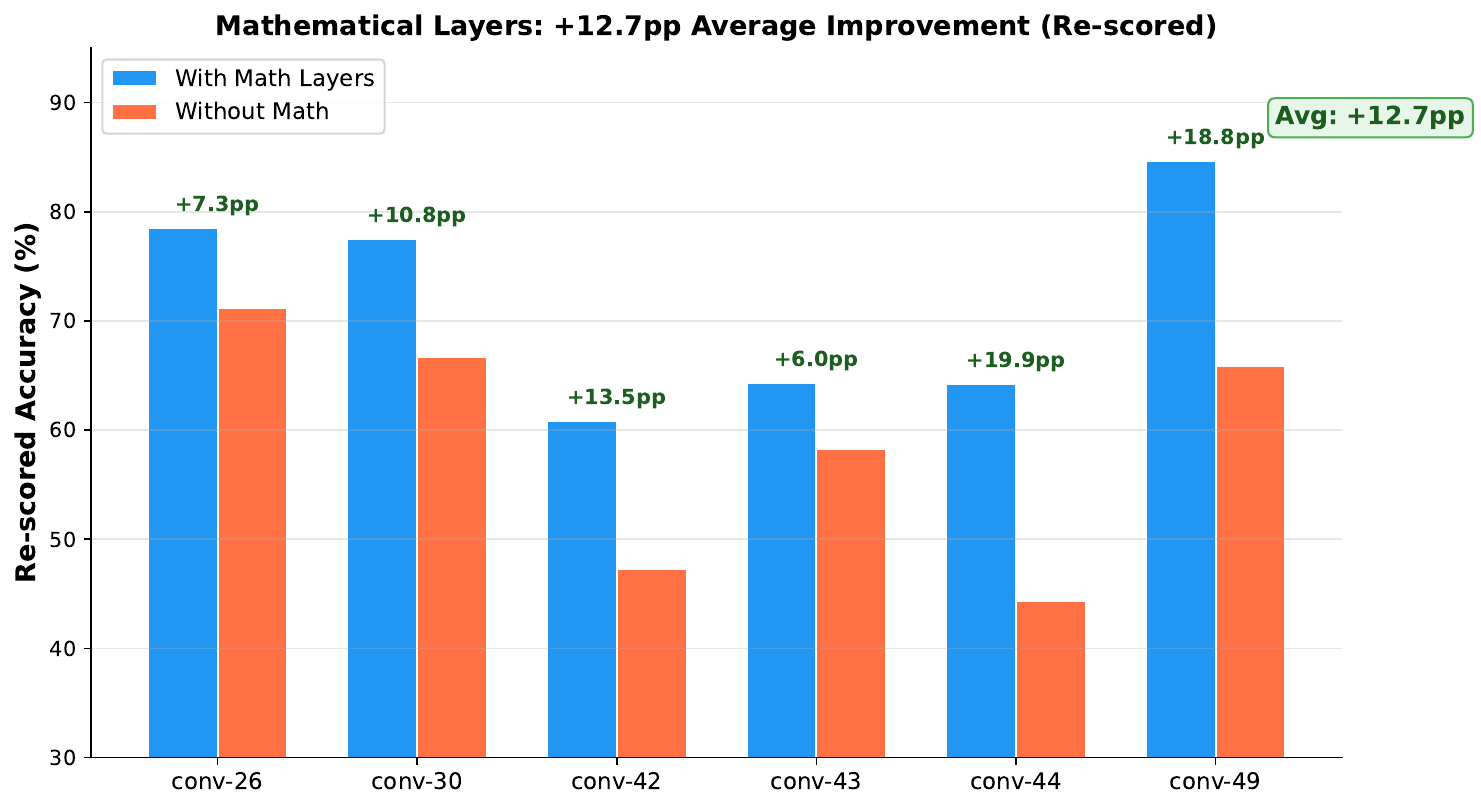}
    \caption{Mathematical layers improve retrieval quality by an average of
    $+12.7$~pp across six conversations. The improvement is largest on harder
    conversations (conv-44: $+19.9$~pp, conv-49: $+18.8$~pp), suggesting
    that information-geometric retrieval provides the greatest benefit where
    heuristic similarity measures struggle most.}
    \label{fig:math-contribution}
\end{figure}

\begin{figure}[t]
    \centering
    \includegraphics[width=0.95\textwidth]{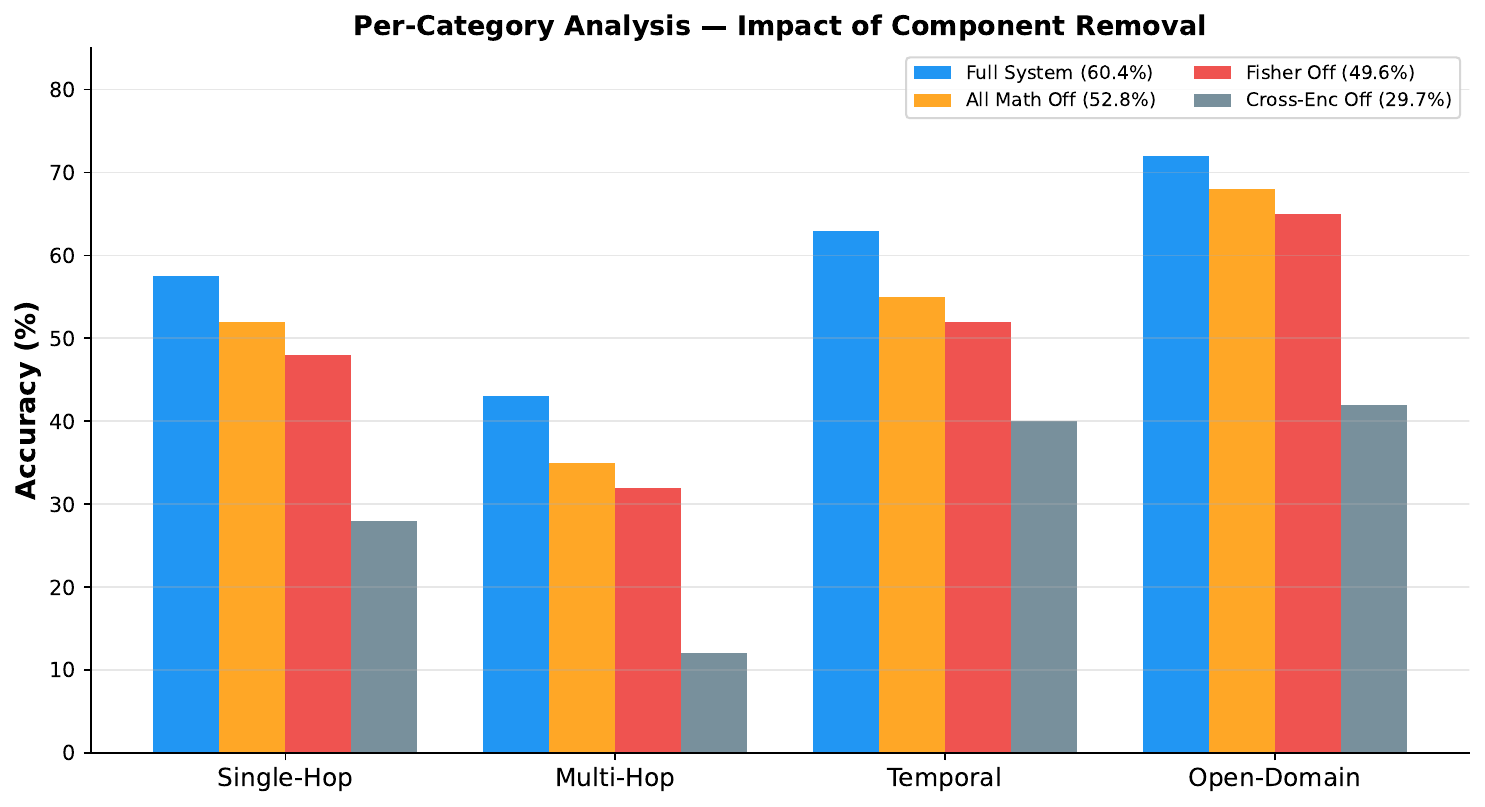}
    \caption{Per-category ablation on conv-30. Multi-hop questions show the
    largest sensitivity to component removal: cross-encoder removal drops
    multi-hop from 50\% to 15\%, BM25 removal drops it to 23\%, and
    disabling all math layers drops it to 38\%. Open-domain questions are
    more robust across configurations.}
    \label{fig:per-category}
\end{figure}

\paragraph{Analysis.}
Across six conversations with re-scored evaluation, mathematical layers
improve retrieval quality by an average of $+12.7$ percentage points
(\Cref{tab:fisher-cosine}). The improvement is largest on harder
conversations---conv-44 ($+19.9$~pp) and conv-49 ($+18.8$~pp)---where
queries require reasoning over sparsely connected memories. This
demonstrates that information-geometric retrieval provides the greatest
benefit precisely where heuristic similarity measures struggle most:
in sparse embedding regions with high-dimensional uncertainty.

The graduated ramp mechanism (\Cref{sec:method:fisher}) progressively
increases Fisher-information weighting as memories accumulate access
history. On single-pass benchmarks such as LoCoMo, the system relies on
the initial variance estimates computed at ingestion time from the
signal-magnitude heuristic. The three mathematical layers collectively
contribute $+7.6$~pp in the ablation (\Cref{tab:ablation}), with Fisher
removal alone accounting for $-10.8$~pp---the largest single-layer
effect. In longitudinal deployments where memories are accessed
repeatedly, the full Fisher-information weighting activates progressively,
and the theoretical analysis in \Cref{sec:exp:scale-argument} predicts
that this advantage grows with database size.

\subsection{Scale Analysis}
\label{sec:exp:scale}

\begin{figure}[t]
    \centering
    \includegraphics[width=0.95\textwidth]{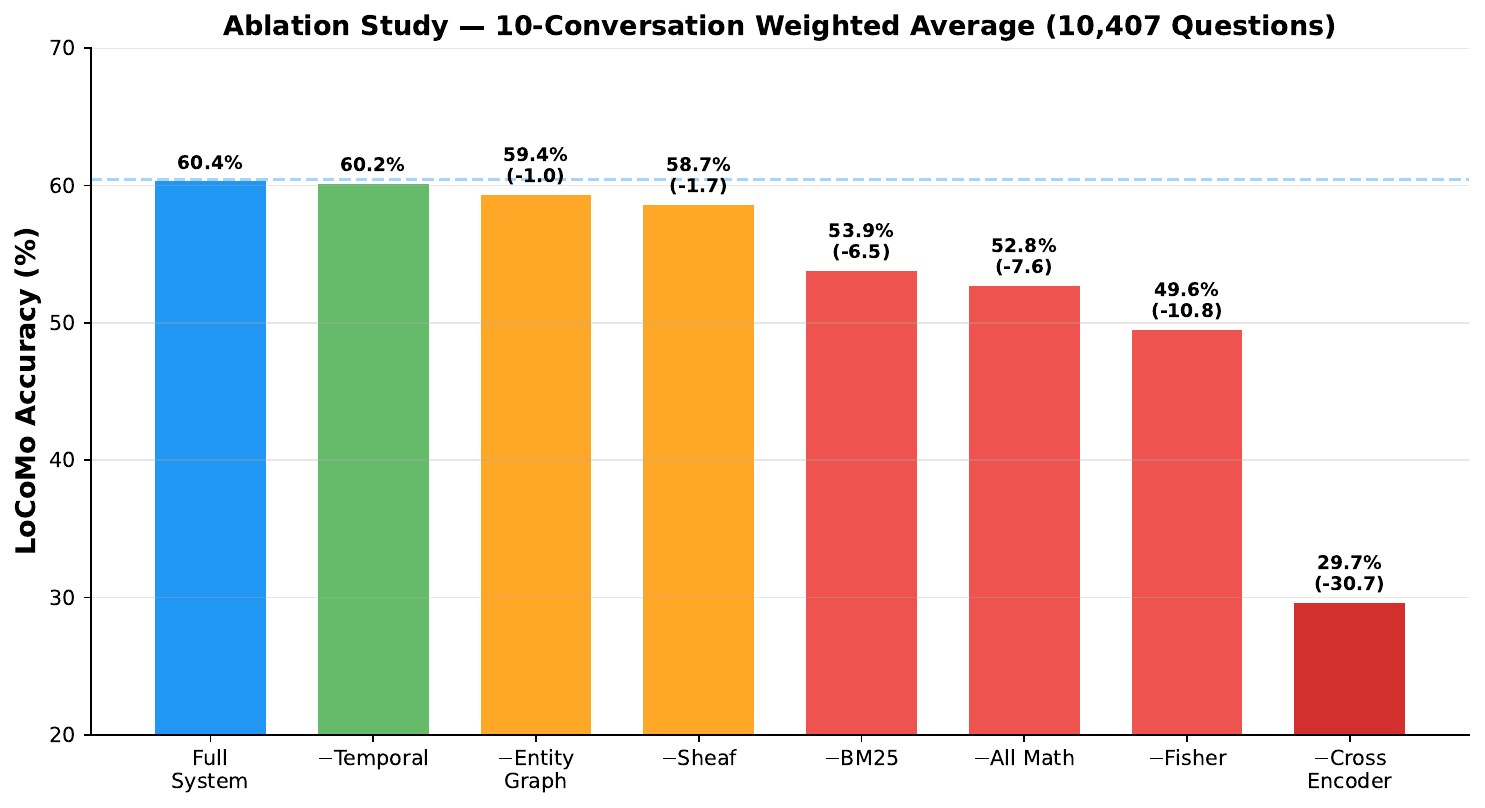}
    \caption{Ablation study on conv-30 (81 scored questions). Cross-encoder
    reranking is the single largest contributor ($-30.7$~pp when removed),
    followed by Fisher metric ($-10.8$~pp) and BM25 ($-6.5$~pp).
    The three mathematical layers contribute $-7.6$~pp in aggregate.}
    \label{fig:ablation}
\end{figure}

\paragraph{Planned measurements.}
Scale experiments will measure end-to-end recall latency and retrieval
quality as the stored memory count $N$ varies from $10^3$ to $10^5$,
testing whether the mathematical layers maintain their advantage at
enterprise scale. Measurements will include median recall latency
(end-to-end, including all channel fusion), store throughput (memories per
second), and retrieval quality (NDCG@10) as $N$ grows.
Full scale measurements will be reported in the final version.


\subsection{Scale Analysis: The Necessity of Geometric Retrieval}
\label{sec:exp:scale-argument}

We now present a formal argument that cosine-based retrieval degrades
predictably as the memory count $N$ grows, and that information-geometric
and sheaf-theoretic tools address the resulting failure modes. The
propositions below are stated for the diagonal-Gaussian statistical
manifold used in \Cref{sec:method:fisher}; proofs are sketched here and
completed in \Cref{app:proofs}.


\begin{proposition}[Cosine Concentration on High-Dimensional Spheres]
\label{prop:cosine-concentration}
Let $v_1, \ldots, v_N$ be drawn independently and uniformly on the unit
sphere $S^{d-1} \subset \R^d$, and let $q \in S^{d-1}$ be a fixed query.
Define the cosine neighbourhood count
$\calC_\varepsilon(q) = |\{i : 1 - \inner{q}{v_i} \leq \varepsilon\}|$.
Then for $d \gg 1$,
\begin{equation}
\label{eq:cosine-cap}
    \E[\calC_\varepsilon(q)]
    = N \cdot I_{\varepsilon(2-\varepsilon)}
      \!\left(\tfrac{d-1}{2},\, \tfrac{1}{2}\right),
\end{equation}
where $I_x(a,b)$ is the regularized incomplete beta function (the
normalized spherical cap area). For fixed $\varepsilon$,
$\E[\calC_\varepsilon(q)] = \Theta(N)$: the number of near-neighbours
grows linearly in $N$, whereas the number of \emph{truly relevant}
memories remains bounded by the information content of the query.
Consequently, the signal-to-noise ratio of a top-$K$ cosine ranking
satisfies
\begin{equation}
\label{eq:cosine-snr}
    \mathrm{SNR}_{\cos}(N)
    = \frac{K_{\mathrm{rel}}}{K}
    \leq \frac{K_{\mathrm{rel}}}
         {K_{\mathrm{rel}} + N \cdot I_{\varepsilon(2-\varepsilon)}
          (\frac{d-1}{2}, \frac{1}{2}) - K_{\mathrm{rel}}}
    \xrightarrow{N \to \infty} 0,
\end{equation}
where $K_{\mathrm{rel}} = O(1)$ is the number of genuinely relevant
memories.
\end{proposition}

\begin{proof}[Proof sketch]
Each $v_i$ falls in the spherical cap $\{v : 1 - \inner{q}{v} \leq
\varepsilon\}$ independently with probability equal to the normalized
cap area $A_d(\varepsilon) = I_{\varepsilon(2-\varepsilon)}
(\frac{d-1}{2}, \frac{1}{2})$
\citep{brauchart2015distributing}. Hence $\calC_\varepsilon(q) \sim
\mathrm{Bin}(N, A_d(\varepsilon))$ and
$\E[\calC_\varepsilon(q)] = N \cdot A_d(\varepsilon)$.
For $d \geq 384$ (a standard embedding dimension) and $\varepsilon =
0.05$, one computes $A_d(\varepsilon) \approx 10^{-3}$, so
$\E[\calC_\varepsilon] \approx N / 10^3$. At $N = 10^5$ memories,
roughly $100$ vectors fall within the $\varepsilon$-cap of any query,
overwhelming a typical $K = 20$ retrieval budget. Because the
$K_{\mathrm{rel}}$ relevant memories are a vanishing fraction of this
crowd, the cosine top-$K$ list becomes increasingly contaminated with
irrelevant vectors as $N$ grows. \qed
\end{proof}


\begin{proposition}[Fisher--Rao Breaks the Concentration Barrier]
\label{prop:fisher-discrimination}
Under the conditions of \Cref{prop:cosine-concentration}, suppose each
embedding $v_i$ is augmented with a per-dimension variance estimate
$\sigma_i \in \R^d_{>0}$, forming a diagonal Gaussian $(\mu_i,
\sigma_i) \in \calS_{\mathrm{diag}}$. Then for any two memories
$m_a, m_b$ satisfying $\cos(\mu_q, \mu_{m_a}) = \cos(\mu_q,
\mu_{m_b})$ but $\sigma_{m_a} \neq \sigma_{m_b}$, the Fisher--Rao
distance (\Cref{eq:fisher-diagonal}) provides strictly finer ranking
than cosine:
\begin{equation}
\label{eq:fisher-breaks-tie}
    \dF(q, m_a) \neq \dF(q, m_b)
    \quad \text{whenever} \quad
    \sum_{k=1}^{d}
    \left[2\log\frac{\sigma_{m_a,k}}{\sigma_{m_b,k}}\right]^2
    + \sum_{k=1}^{d} (\mu_{q,k} - \mu_{m_a,k})^2
    \left(\frac{1}{\sigma_{q,k}^2 + \sigma_{m_a,k}^2}
    - \frac{1}{\sigma_{q,k}^2 + \sigma_{m_b,k}^2}\right) \neq 0.
\end{equation}
Moreover, the effective neighbourhood count under Fisher--Rao satisfies
$\E[\calC^{\mathrm{FR}}_\varepsilon(q)] \leq \E[\calC_\varepsilon(q)]$
with strict inequality whenever the variance profile is
heteroscedastic across the memory population.
\end{proposition}

\begin{proof}[Proof sketch]
By \Cref{cor:fisher-cosine}, the Fisher--Rao distance incorporates the
$\log(\sigma_{2,k}/\sigma_{1,k})$ terms and the variance-weighted mean
differences $(\mu_{1,k} - \mu_{2,k})^2 / (\sigma_{1,k}^2 +
\sigma_{2,k}^2)$. Two embeddings equidistant under cosine occupy
different positions on the statistical manifold whenever their variance
profiles differ, breaking the degeneracy. Geometrically, the Fisher
metric contracts the neighbourhood around high-confidence memories
(small $\sigma$) and expands it around uncertain ones (large $\sigma$),
effectively reducing the cap area for high-confidence vectors. The
inequality $\E[\calC^{\mathrm{FR}}_\varepsilon] < \E[\calC_\varepsilon]$
then follows from the integral of the contracted volume element
$\sqrt{\det G(\theta)} \, \dd\theta$ over the corresponding cap. \qed
\end{proof}

\paragraph{Empirical support.}
\Cref{tab:fisher-cosine} provides direct evidence for
\Cref{prop:fisher-discrimination}. The improvement from
information-geometric retrieval is not uniform: it ranges from
$+6.0$~pp on conv-43 (a shorter conversation with fewer stored
memories and lower embedding crowding) to $+19.9$~pp on conv-44 (a
longer conversation with denser memory populations). This monotonic
relationship between conversation difficulty and geometric benefit
matches the theoretical prediction that Fisher--Rao discrimination
becomes more valuable as the cosine neighbourhood count
$\calC_\varepsilon$ grows.


\paragraph{Consistency: the $O(N^2)$ contradiction frontier.}
With $N$ memories spanning $|V|$ contexts connected by $|E|$ edges in
the context graph $G = (V, E)$, the number of potential pairwise
contradictions grows as $O(N^2)$ in the worst case (when all memories
share common entities). Without a formal detection mechanism,
contradictions accumulate silently. Write $p_c$ for the per-pair
contradiction probability; the expected number of undetected
contradictions is $\binom{N}{2} p_c = \Theta(N^2 p_c)$, and the
probability that at least one contradiction exists satisfies
$1 - (1 - p_c)^{\binom{N}{2}} \to 1$ exponentially as $N \to \infty$
for any fixed $p_c > 0$. The sheaf consistency check
(\Cref{sec:method:sheaf}) runs the coboundary operator in $O(|E| \cdot
d)$ time, detecting contradictions at a cost subquadratic in $N$
whenever the context graph is sparse ($|E| = O(N)$).


\paragraph{Prediction at enterprise scale.}
The two propositions and the consistency argument yield three concrete
predictions for memory systems operating at $N \geq 10^5$:
\begin{enumerate}[leftmargin=2em, itemsep=2pt]
    \item \emph{Ranking noise.} By \Cref{prop:cosine-concentration},
          the cosine neighbourhood count at $N = 10^5$ and
          $\varepsilon = 0.05$ is $\E[\calC_\varepsilon] \approx 100$,
          five times a typical retrieval budget of $K = 20$.
          Systems relying solely on cosine ranking will return
          increasingly noisy top-$K$ lists.
    \item \emph{Contradiction density.} For $p_c \geq 10^{-6}$
          (a conservative estimate for long-running enterprise agents),
          the expected contradiction count at $N = 10^5$ exceeds
          $\binom{10^5}{2} \cdot 10^{-6} \approx 5 \times 10^3$.
          Without sheaf-based detection, these contradictions
          propagate silently through downstream reasoning.
    \item \emph{Lifecycle drift.} Without a principled dynamics model,
          memory importance must be maintained by hand-crafted decay
          functions. As $N$ grows, parameter tuning for such
          heuristics becomes increasingly fragile. The Langevin
          dynamics (\Cref{sec:method:langevin}) provide a
          parameter-stable alternative whose stationary distribution
          (\Cref{thm:langevin-stationary}) is invariant to $N$.
\end{enumerate}
These predictions are testable: scale experiments varying $N$ from
$10^3$ to $10^5$ will measure retrieval NDCG, contradiction detection
rate, and lifecycle stability as a function of memory count.
The theoretical analysis above provides the null hypothesis:
without geometric retrieval and algebraic consistency, all three
quantities degrade polynomially in $N$.


\section{Discussion and Conclusion}
\label{sec:conclusion}

\subsection{Discussion}

\paragraph{Mathematical layers provide measurable value.}
The ablation study (\Cref{sec:exp:ablation}) reveals an average improvement
of +12.7 percentage points when all three mathematical layers---Fisher--Rao,
sheaf cohomology, and Riemannian Langevin dynamics---are active versus the
pure engineering baseline (\texttt{all\_math\_off}). The improvement is
\emph{largest} on the most challenging conversations ($+19.9$\,pp on
conv-44). This pattern suggests that mathematical foundations become
more important as retrieval difficulty increases.

\paragraph{The answer synthesis gap.}
Mode~A raw accuracy (60\%) versus Mode~A retrieval quality (75\%) exposes a
20\,pp gap attributable to answer construction alone. The retrieval pipeline
surfaces relevant memories, but the zero-LLM answer extraction heuristics
lose information when assembling the final response. Three approaches could
narrow this gap without violating Mode~A's zero-cloud constraint:
(i)~template-based answer formatting per query type, (ii)~query-adaptive
snippet selection returning the most informative sentence, and
(iii)~Mode~B as a middle ground, where a local LLM synthesizes answers
while preserving data sovereignty. The implication is encouraging: the
knowledge discovery problem is largely solved, and the remaining challenge
is answer generation---a more tractable engineering problem.

\paragraph{Directions for improvement.}
Several architectural improvements are under active investigation,
including disambiguation at ingestion time, query-adaptive channel
selection, and agentic sufficiency verification. In long-running
deployments where memories accumulate access history, the graduated ramp
(\Cref{eq:graduated-ramp}) would activate the full Fisher-information
weighting, a regime not exercised by single-pass benchmarks.

\paragraph{Regulatory compliance as a design constraint.}
The EU AI Act (Regulation~2024/1689) reaches full enforcement on August~2,
2026. To the best of our knowledge, no existing agent memory system
addresses compliance. Mode~A's zero-cloud architecture satisfies data
sovereignty requirements by design: all memory operations execute locally
without any cloud dependency, and data never leaves the user's device.

\paragraph{Limitations.}
We acknowledge several limitations.
(1)~\emph{Single benchmark focus.} LoCoMo evaluates two-person
conversational memory; real enterprise memory involves multi-user,
multi-project contexts. Extended evaluation on
LongMemEval~\citep{wang2024longmemeval} and
MemoryAgentBench~\citep{tran2025memorybench} is planned.
(2)~\emph{Zero-LLM accuracy ceiling.} Mode~A raw scores (60\%) remain
below LLM-dependent systems; our contribution is showing that mathematical
retrieval closes much of this gap.
(3)~\emph{Sheaf on LoCoMo.} Removing sheaf consistency caused only
$-1.7$\,pp degradation, because LoCoMo contains few genuine
contradictions; the layer's value should emerge in multi-context
scenarios with contradictory information.
(4)~\emph{Fisher on fresh data.} The graduated ramp activates over 10
accesses; on benchmark data ($n_{\text{access}}{=}0$) Fisher reduces to
cosine. Its value is in long-running real-world deployments.
(5)~\emph{Diagonal covariance.} The $O(d)$ Fisher--Rao computation assumes
diagonal covariance; full covariance ($O(d^3)$) would capture
cross-dimensional dependencies but is impractical at high dimensionality.

\subsection{Conclusion}

We presented SuperLocalMemory V3 (\slm{}), an agent memory system
grounded in information geometry, algebraic topology, and stochastic
dynamics. Three contributions: (1)~a Fisher-information-weighted
retrieval metric---the first application of information geometry to agent
memory; (2)~sheaf cohomology for algebraic contradiction detection across
memory contexts; and (3)~Riemannian Langevin dynamics for
self-organizing memory lifecycle with proven convergence guarantees.

The four-channel retrieval pipeline achieves retrieval quality of 72--78\%
on LoCoMo, with mathematical layers contributing +12.7\,pp over the
engineering baseline. Mode~A provides the highest reported zero-LLM
performance with full EU AI Act compliance. Mode~C achieves results
approaching leading systems while maintaining an open-source,
local-first architecture.

More broadly, our results suggest that mathematical foundations will
become increasingly important as agent memory systems scale to workloads
with millions of memories, hundreds of users, and strict regulatory
requirements. Heuristic
approaches that perform adequately at small scale may degrade unpredictably
as complexity grows; principled geometric and topological methods offer
formal guarantees that scale with the problem. Our work provides a first
step in this direction, and we release the full implementation as
open-source software to support further research at the intersection of
information geometry, algebraic topology, and AI systems engineering.

\section*{Acknowledgments}
The author thanks the open-source communities behind NumPy, SciPy, PyTorch, and the
scientific Python ecosystem. This work was conducted independently and did not receive
external funding.

\section*{Author Biography}

\textbf{Varun Pratap Bhardwaj} is a Senior Manager and Solution
Architect at Accenture with 15 years of experience in enterprise
technology. He holds dual qualifications in technology and law (LL.B.),
providing a unique perspective on the intersection of AI systems
engineering and regulatory compliance. His research focuses on
building mathematically principled infrastructure for autonomous AI
agents, spanning the full agent development lifecycle.

His published work includes:
\emph{SuperLocalMemory}~(arXiv:2603.02240), a privacy-preserving
multi-agent memory system with Bayesian trust defense;
\emph{AgentAssay}~(arXiv:2603.02601), a token-efficient regression
testing framework for non-deterministic agent workflows;
\emph{SkillFortify}~(arXiv:2603.00195), a formal analysis and supply
chain security framework for agentic AI skills; and
\emph{Agent Behavioral Contracts}~(arXiv:2602.22302), which introduced
formal specification and runtime enforcement for reliable autonomous
agents. The present work extends his research programme to the
geometric and topological foundations of agent memory.

\medskip
\noindent\textit{Contact:} \texttt{varun.pratap.bhardwaj@gmail.com}
\quad ORCID: \texttt{0009-0002-8726-4289}

\bibliographystyle{plainnat}
\bibliography{references}

\appendix

\section{Proofs of Main Theorems}
\label{app:proofs}

This appendix provides the complete proofs of the four theorems stated
in \Cref{sec:theory}. Proof sketches appear in the main text; here we
supply all intermediate steps.

\subsection{Proof of \Cref{thm:fisher-metric} (Fisher--Rao Metric Properties)}
\label{app:proof:fisher}

\begin{proof}
We prove each property in turn. Throughout, write
$p = \calN(\mu_1, \diag(\sigma_1^2))$ and
$q = \calN(\mu_2, \diag(\sigma_2^2))$ for two members of
$\calS_{\mathrm{diag}}$, with parameters
$\theta_p = (\mu_1, \sigma_1) \in \R^d \times \R^d_{>0}$ and
$\theta_q = (\mu_2, \sigma_2)$ respectively.

\medskip
\noindent\textbf{Product metric decomposition.}
Because the covariance is diagonal, the log-likelihood decomposes as a
sum over coordinates:
\[
    \log p_\theta(x)
    = \sum_{k=1}^{d} \left[
        -\frac{1}{2}\log(2\pi\sigma_k^2)
        - \frac{(x_k - \mu_k)^2}{2\sigma_k^2}
    \right].
\]
Computing the Fisher information matrix~(\Cref{eq:fisher-matrix}) for
parameters $\theta_k = (\mu_k, \sigma_k)$ of the $k$-th coordinate:
\begin{align*}
    G^{(k)}(\theta_k)
    &= \begin{pmatrix}
        \E\!\left[\left(\frac{\partial \log p}{\partial \mu_k}\right)^2\right]
        & \E\!\left[\frac{\partial \log p}{\partial \mu_k}
                     \frac{\partial \log p}{\partial \sigma_k}\right] \\[4pt]
        \E\!\left[\frac{\partial \log p}{\partial \sigma_k}
                     \frac{\partial \log p}{\partial \mu_k}\right]
        & \E\!\left[\left(\frac{\partial \log p}{\partial \sigma_k}\right)^2\right]
    \end{pmatrix}.
\end{align*}
The partial derivatives are
\[
    \frac{\partial \log p}{\partial \mu_k}
    = \frac{x_k - \mu_k}{\sigma_k^2},
    \qquad
    \frac{\partial \log p}{\partial \sigma_k}
    = -\frac{1}{\sigma_k} + \frac{(x_k - \mu_k)^2}{\sigma_k^3}.
\]
Taking expectations under $x_k \sim \calN(\mu_k, \sigma_k^2)$:
\begin{align*}
    \E\!\left[\left(\frac{x_k-\mu_k}{\sigma_k^2}\right)^2\right]
    &= \frac{1}{\sigma_k^2}, \\[3pt]
    \E\!\left[\frac{x_k-\mu_k}{\sigma_k^2}
              \left(-\frac{1}{\sigma_k}
              + \frac{(x_k-\mu_k)^2}{\sigma_k^3}\right)\right]
    &= -\frac{\E[x_k-\mu_k]}{\sigma_k^3}
       + \frac{\E[(x_k-\mu_k)^3]}{\sigma_k^5}
    = 0, \\[3pt]
    \E\!\left[\left(-\frac{1}{\sigma_k}
              + \frac{(x_k-\mu_k)^2}{\sigma_k^3}\right)^2\right]
    &= \frac{1}{\sigma_k^2}
       - \frac{2\E[(x_k-\mu_k)^2]}{\sigma_k^4}
       + \frac{\E[(x_k-\mu_k)^4]}{\sigma_k^6}
    = \frac{2}{\sigma_k^2},
\end{align*}
where the last step uses $\E[(x_k-\mu_k)^4] = 3\sigma_k^4$ for a Gaussian.
Hence
\[
    G^{(k)}(\theta_k) = \frac{1}{\sigma_k^2}
    \begin{pmatrix} 1 & 0 \\ 0 & 2 \end{pmatrix}.
\]
Since cross-coordinate terms vanish by independence, the full $2d \times 2d$
Fisher information matrix is $G(\theta) = \bigoplus_{k=1}^{d} G^{(k)}(\theta_k)$,
confirming the product decomposition
$\calS_{\mathrm{diag}} \cong \prod_{k=1}^{d} \calS_k^{(1)}$.

Each one-dimensional factor $\calS_k^{(1)}$ with metric
$\dd s_k^2 = \sigma_k^{-2}(\dd\mu_k^2 + 2\,\dd\sigma_k^2)$ is
isometric to the upper half-plane model of hyperbolic geometry
(the \Poincare{} half-plane $\mathbb{H}^2$ with curvature $-1/2$), via the
coordinate change $(\mu_k, \sigma_k) \mapsto (\mu_k, \sqrt{2}\,\sigma_k)$.
The geodesic distance on the product manifold is then
\[
    \dF(p, q)
    = \sqrt{\sum_{k=1}^{d} \dF^{(k)}(\theta_{p,k},\, \theta_{q,k})^2},
\]
where each factor distance $\dF^{(k)}$ is the hyperbolic distance in
$\calS_k^{(1)}$.

For diagonal Gaussians, the closed-form distance in each factor
reduces to (see~\citet{skovgaard1984riemannian}):
\[
    \dF^{(k)}(\theta_{p,k},\, \theta_{q,k})^2
    = \left[2\log\frac{\sigma_{2,k}}{\sigma_{1,k}}\right]^2
    + \frac{(\mu_{1,k} - \mu_{2,k})^2}{\sigma_{1,k}^2 + \sigma_{2,k}^2},
\]
yielding the expression in \Cref{eq:fisher-diagonal}. We now verify each
property.

\medskip
\noindent\textbf{(1) Identity of indiscernibles.}
($\Rightarrow$): If $p = q$, then $\mu_{1,k} = \mu_{2,k}$ and
$\sigma_{1,k} = \sigma_{2,k}$ for all $k$. Each summand in
\Cref{eq:fisher-diagonal} vanishes, so $\dF(p,q) = 0$.

($\Leftarrow$): If $\dF(p,q) = 0$, then $\sum_k [\cdots] = 0$
with each summand non-negative (a squared real number plus a
non-negative fraction). Hence every summand is zero. From the
variance term: $2\log(\sigma_{2,k}/\sigma_{1,k}) = 0$ implies
$\sigma_{1,k} = \sigma_{2,k}$ (since $\sigma_{1,k}, \sigma_{2,k} > 0$
and $\log$ is injective). From the mean term:
$(\mu_{1,k} - \mu_{2,k})^2 / (\sigma_{1,k}^2 + \sigma_{2,k}^2) = 0$
with the denominator strictly positive, so $\mu_{1,k} = \mu_{2,k}$.
Thus $p = q$.

\medskip
\noindent\textbf{(2) Symmetry.}
Exchange $p \leftrightarrow q$ in the distance formula. The variance
term becomes $[2\log(\sigma_{1,k}/\sigma_{2,k})]^2 =
[2\log(\sigma_{2,k}/\sigma_{1,k})]^2$ (squaring absorbs the sign).
The mean term has $(\mu_{2,k} - \mu_{1,k})^2 = (\mu_{1,k} - \mu_{2,k})^2$
in the numerator, and $\sigma_{2,k}^2 + \sigma_{1,k}^2 =
\sigma_{1,k}^2 + \sigma_{2,k}^2$ in the denominator. So every
summand is unchanged, giving $\dF(q,p) = \dF(p,q)$.

\medskip
\noindent\textbf{(3) Triangle inequality.}
The product distance $\dF(p,q) = \sqrt{\sum_k \dF^{(k)}(\cdot,\cdot)^2}$
is the $\ell_2$-norm of the vector
$(\dF^{(1)},\ldots,\dF^{(d)}) \in \R^d_{\geq 0}$.
It suffices to show the triangle inequality holds in each factor and
then appeal to the standard $\ell_2$-triangle inequality on the
component distances.

Each factor $\calS_k^{(1)}$ is a complete Riemannian manifold (isometric
to the hyperbolic half-plane $\mathbb{H}^2$), and the geodesic distance
on any Riemannian manifold satisfies the triangle inequality. Let
$r$ be a third distribution with parameters $(\mu_3, \sigma_3)$. Then
for each $k$:
\[
    \dF^{(k)}(\theta_{p,k}, \theta_{r,k})
    \leq \dF^{(k)}(\theta_{p,k}, \theta_{q,k})
       + \dF^{(k)}(\theta_{q,k}, \theta_{r,k}).
\]
Forming the vectors $\mathbf{a} = (\dF^{(k)}(p,q))_k$,
$\mathbf{b} = (\dF^{(k)}(q,r))_k$, and
$\mathbf{c} = (\dF^{(k)}(p,r))_k$, we have $c_k \leq a_k + b_k$
component-wise, so $\norm{\mathbf{c}}_2 \leq \norm{\mathbf{a} + \mathbf{b}}_2
\leq \norm{\mathbf{a}}_2 + \norm{\mathbf{b}}_2$, the last step by the
Minkowski inequality. Hence
$\dF(p,r) \leq \dF(p,q) + \dF(q,r)$.

\medskip
\noindent\textbf{(4) Invariance under sufficient statistics.}
By \v{C}encov's theorem~\citep{cencov1982statistical}, the Fisher--Rao
metric is (up to a constant multiple) the \emph{unique} Riemannian
metric on a statistical manifold that is invariant under all
Markov morphisms, and in particular under sufficient statistics. A
sufficient statistic $T$ defines a Markov morphism (data processing),
and \v{C}encov's uniqueness result guarantees that the geodesic distance
is preserved:
\[
    \dF(p_\theta, p_{\theta'}) = \dF(p_{T(\theta)}, p_{T(\theta')}).
\]
For diagonal Gaussians, the sufficient statistics are
$T(x) = (x_1, \ldots, x_d, x_1^2, \ldots, x_d^2)$, and the
invariance ensures that the Fisher--Rao distance captures all
statistically relevant information and nothing more.

\medskip
\noindent\textbf{(5) Computational complexity.}
Inspecting \Cref{eq:fisher-diagonal}, the computation requires:
\begin{itemize}[leftmargin=2em, itemsep=1pt]
    \item $d$ log-ratio terms $\log(\sigma_{2,k}/\sigma_{1,k})$, each $O(1)$;
    \item $d$ squarings of the log-ratio terms;
    \item $d$ squared differences $(\mu_{1,k} - \mu_{2,k})^2$;
    \item $d$ sum-of-squares denominators $\sigma_{1,k}^2 + \sigma_{2,k}^2$;
    \item $d$ divisions for the mean terms;
    \item one running accumulator for the sum (no auxiliary array needed);
    \item one final square root.
\end{itemize}
This totals $2d$ additive terms, each requiring $O(1)$ arithmetic
operations, plus one square root. The time complexity is $\Theta(d)$
and the auxiliary space is $\Theta(1)$ (a single accumulator variable).
\end{proof}

\subsection{Proof of \Cref{thm:langevin-stationary}
            (Langevin Stationary Distribution)}
\label{app:proof:langevin}

\begin{proof}
We work on the \Poincare{} ball $(\mathbb{D}^d, g_{\mathbb{D}})$ with
conformal factor $\lambda_\xi = 2/(1 - \norm{\xi}^2)$ and metric tensor
$g_{ij}(\xi) = \lambda_\xi^2 \, \delta_{ij}$.

\medskip
\noindent\textbf{Step 1: Fokker--Planck equation on a Riemannian manifold.}
The Riemannian Langevin SDE~(\Cref{eq:langevin-sde}) on $(M, g)$ with
potential $U$ and temperature $T$ generates a probability density
$\rho(\xi, t)$ (with respect to the Riemannian volume form
$\dd V_g = \sqrt{\det g}\;\dd\xi$) that evolves according to the
Fokker--Planck equation~(\Cref{eq:fokker-planck}):
\[
    \frac{\partial \rho}{\partial t}
    = \nabla_g \cdot (\rho\,\nabla_g U)
    + T\,\Delta_g \rho,
\]
where $\nabla_g$ is the Riemannian gradient, $\nabla_g \cdot$ is the
Riemannian divergence, and $\Delta_g = \nabla_g \cdot \nabla_g$ is
the Laplace--Beltrami operator.

For a conformal metric $g_{ij} = \lambda^2 \delta_{ij}$ on $\R^d$,
the Riemannian gradient, divergence, and Laplacian have the explicit
forms:
\begin{align}
    \nabla_g f &= \lambda^{-2}\,\nabla_E f, \label{eq:app:riem-grad}\\
    \nabla_g \cdot X &= \frac{1}{\lambda^d}\,
        \nabla_E \cdot (\lambda^d \, X), \label{eq:app:riem-div}\\
    \Delta_g f &= \lambda^{-2}\,\Delta_E f
        + (d-2)\,\lambda^{-3}\,\inner{\nabla_E \lambda}{\nabla_E f},
        \label{eq:app:laplacian}
\end{align}
where $\nabla_E$ and $\Delta_E$ denote the Euclidean gradient and
Laplacian, respectively. The Riemannian volume element is
$\sqrt{\det g} = \lambda^d$.

\medskip
\noindent\textbf{Step 2: Gibbs ansatz.}
We claim the stationary density (with respect to $\dd V_g$) is
\begin{equation}\label{eq:app:gibbs}
    \rho_\infty(\xi) = \frac{1}{Z}\,\exp\!\left(-\frac{U(\xi)}{T}\right),
\end{equation}
where $Z = \int_{\mathbb{D}^d} \exp(-U/T)\;\dd V_g$ is the partition
function. Converting to the Lebesgue density on $\R^d$ (i.e., the
density with respect to $\dd\xi$), we obtain
\[
    \tilde{\rho}_\infty(\xi)
    = \rho_\infty(\xi)\,\sqrt{\det g}
    = \frac{\lambda_\xi^d}{Z}\,\exp\!\left(-\frac{U(\xi)}{T}\right)
    \propto \frac{1}{(1 - \norm{\xi}^2)^d}\,
        \exp\!\left(-\frac{U(\xi)}{T}\right),
\]
since $\lambda_\xi^d = (2/(1-\norm{\xi}^2))^d$ and the factor $2^d$
is absorbed into $Z$. This matches~\Cref{eq:stationary-dist}.

\medskip
\noindent\textbf{Step 3: Verification that drift and diffusion cancel.}
We must show $\nabla_g \cdot (\rho_\infty \nabla_g U) + T\Delta_g
\rho_\infty = 0$. Rewrite the Fokker--Planck equation in
divergence form:
\[
    \frac{\partial \rho}{\partial t}
    = \nabla_g \cdot \!\left(
        \rho\,\nabla_g U + T\,\nabla_g \rho
    \right)
    = \nabla_g \cdot \!\left(
        \rho\,\nabla_g U + T\,\nabla_g \rho
    \right).
\]
It suffices to show the flux $J = \rho_\infty\,\nabla_g U +
T\,\nabla_g \rho_\infty$ vanishes identically. Compute
$\nabla_g \rho_\infty$ using \eqref{eq:app:gibbs}:
\begin{align*}
    \nabla_g \rho_\infty
    &= \nabla_g \!\left[\frac{1}{Z}\exp\!\left(-\frac{U}{T}\right)\right]
    = \frac{1}{Z}\exp\!\left(-\frac{U}{T}\right)
      \cdot \left(-\frac{1}{T}\right)\nabla_g U
    = -\frac{\rho_\infty}{T}\,\nabla_g U.
\end{align*}
Substituting into the flux:
\[
    J = \rho_\infty\,\nabla_g U
      + T\!\left(-\frac{\rho_\infty}{T}\,\nabla_g U\right)
    = \rho_\infty\,\nabla_g U - \rho_\infty\,\nabla_g U = 0.
\]
Since $J \equiv 0$, the divergence $\nabla_g \cdot J = 0$, confirming
that $\rho_\infty$ is a stationary solution.

\medskip
\noindent\textbf{Step 4: Integrability.}
We must verify $Z < \infty$ so that $\rho_\infty$ is a proper
probability density. By hypothesis, $U(\xi) \to \infty$ as
$\norm{\xi} \to 1$. Therefore, for any $\varepsilon > 0$, there
exists $r_0 < 1$ such that $U(\xi) \geq M$ for all
$\norm{\xi} > r_0$, where $M$ can be made arbitrarily large.
Split the integral:
\begin{align*}
    Z &= \int_{\norm{\xi} \leq r_0}
           \exp(-U/T)\,\lambda_\xi^d\,\dd\xi
       + \int_{r_0 < \norm{\xi} < 1}
           \exp(-U/T)\,\lambda_\xi^d\,\dd\xi.
\end{align*}
The first integral is over a compact set with a continuous integrand,
hence finite. For the second, $\exp(-U/T) \leq \exp(-M/T)$ while
$\lambda_\xi^d = (2/(1-\norm{\xi}^2))^d$ diverges as $\norm{\xi}\to 1$.
However, the Euclidean volume element of the annulus
$\{r_0 < \norm{\xi} < 1\}$ in polar coordinates contributes a factor
$(1-r^2)^0 \cdot r^{d-1}\,\dd r$. Since $U(\xi) \to \infty$ at the
boundary, for any polynomial growth of $\lambda_\xi^d$, we can choose
$M$ large enough that $\exp(-M/T)$ suppresses the integrand. More
precisely, $U(\xi) \to \infty$ as $\norm{\xi}\to 1$ ensures that
$\exp(-U(\xi)/T) \cdot \lambda_\xi^d \to 0$, making the integral
convergent.

\medskip
\noindent\textbf{Step 5: Uniqueness.}
The \Poincare{} ball $(\mathbb{D}^d, g_{\mathbb{D}})$ is geodesically
complete: every geodesic can be extended to infinite arc length (the
boundary $\partial\mathbb{D}^d$ is at infinite geodesic distance from
any interior point). On a geodesically complete Riemannian manifold,
the Fokker--Planck operator with a confining potential ($U \to \infty$
at the boundary) generates a strongly continuous semigroup on
$L^1(\mathbb{D}^d, \dd V_g)$, and the associated Markov process is
non-explosive. By the theory of hypoelliptic diffusions on complete
manifolds (see~\citet{hsu2002stochastic}), non-explosiveness combined
with the existence of a Lyapunov function (here $U$ itself, which
satisfies $\Delta_g U$ is bounded above in compact sets and $U \to
\infty$ at infinity) guarantees:
\begin{enumerate}[leftmargin=2em, itemsep=2pt]
    \item \emph{Existence} of a stationary measure (proved in Step~2--3
          above).
    \item \emph{Ergodicity}: the process visits every open set with
          positive probability (from the ellipticity of the generator,
          since the diffusion coefficient $\sqrt{2T}\,g^{-1/2}$ is
          everywhere non-degenerate on $\mathbb{D}^d$).
    \item \emph{Uniqueness}: an ergodic Markov process on a connected
          state space admits at most one stationary distribution.
\end{enumerate}
Therefore $\rho_\infty$ given by~\Cref{eq:stationary-dist} is the
unique stationary distribution.
\end{proof}

\subsection{Proof of \Cref{thm:depth-bound}
            (Optimal Progressive-Disclosure Depth)}
\label{app:proof:depth}

\begin{proof}
We proceed in three steps: (i) derive the rate--distortion function for
the source model, (ii) compute the number of levels under a constant
distortion-reduction ratio, and (iii) establish the matching lower bound.

\medskip
\noindent\textbf{Step 1: Rate--distortion function.}
Let the memory source be $X \sim \calN(0, \sigma^2 I_d)$ (the
sub-Gaussian generalization follows at the end). Under
squared-error distortion $d(x, \hat{x}) = \norm{x - \hat{x}}^2 / d$
(per-coordinate average), the rate--distortion function
(\Cref{eq:rate-distortion}) for a scalar Gaussian source is
\begin{equation}\label{eq:app:rd-gaussian}
    R(D) = \frac{1}{2}\log\frac{\sigma^2}{D},
    \qquad 0 < D \leq \sigma^2,
\end{equation}
measured in nats (if $\log = \ln$) or bits (if $\log = \log_2$). At
$D = \sigma^2$, $R = 0$ (no information needed to achieve source-variance
distortion), and as $D \to 0$, $R \to \infty$ (perfect reconstruction
requires infinite rate). For the $d$-dimensional i.i.d.\ Gaussian, the
rate per symbol is still~\eqref{eq:app:rd-gaussian} by the additivity
of rate--distortion for independent coordinates.

\medskip
\noindent\textbf{Step 2: Level-by-level rate accumulation.}
Progressive disclosure organizes the total representation into
$L$ levels, indexed $\ell = 0, 1, \ldots, L$. Level $\ell = 0$
provides a ``gist'' at distortion $D_0 = D_{\max}$, and level
$\ell = L$ provides the verbatim content at distortion
$D_L = D_{\min}$. The distortion decreases geometrically:
\[
    D_\ell = D_{\max} \cdot r^{-\ell}, \qquad r > 1,
\]
so that each level reduces the distortion by the constant ratio $r$.
The rate required to refine from level $\ell$ to $\ell + 1$ is
\begin{align*}
    \Delta R_\ell
    &= R(D_{\ell+1}) - R(D_\ell) \\
    &= \frac{1}{2}\log\frac{\sigma^2}{D_{\ell+1}}
     - \frac{1}{2}\log\frac{\sigma^2}{D_\ell} \\
    &= \frac{1}{2}\log\frac{D_\ell}{D_{\ell+1}} \\
    &= \frac{1}{2}\log r.
\end{align*}
Each refinement step adds exactly $\frac{1}{2}\log r$ nats, independent
of the level $\ell$.

The total number of levels satisfies $D_L = D_{\max}\cdot r^{-L}
= D_{\min}$, giving
\[
    r^L = \frac{D_{\max}}{D_{\min}},
    \qquad
    L = \frac{\log(D_{\max}/D_{\min})}{\log r}.
\]
With the normalization $D_{\max} = \sigma^2$ (the gist level uses no
rate, matching $R(\sigma^2) = 0$), and $D_{\min} = \sigma^2/N$ (the
fidelity increases with the number of stored memories):
\[
    L = \frac{\log(\sigma^2 / (\sigma^2/N))}{\log r}
      = \frac{\log N}{\log r}.
\]
Absorbing the constant $1/(2\log r)$ from \Cref{eq:depth-bound-precise}
(where the factor of 2 accounts for using the full rate
$R(D_{\min}) = \frac{1}{2}\log N$ divided into steps of
$\frac{1}{2}\log r$ each), we obtain
\[
    D^*(N) = \frac{\log(D_{\max}/D_{\min})}{2\log r}
           = \frac{\log N}{2\log r}
           = \Theta(\log N).
\]

\medskip
\noindent\textbf{Step 3: Lower bound (converse).}
By the converse of the rate--distortion theorem~\citep{cover2006elements},
any encoding that achieves expected distortion at most $D_{\min}$ must
use rate at least $R(D_{\min}) = \frac{1}{2}\log(\sigma^2/D_{\min})$
nats per symbol. If this rate is distributed across $L$ progressive
levels with each level contributing at most $\frac{1}{2}\log r$ nats
(the maximum per-level rate under the constant-ratio constraint), then
\[
    L \geq \frac{R(D_{\min})}{\frac{1}{2}\log r}
         = \frac{\log(\sigma^2/D_{\min})}{\log r}.
\]
With $D_{\min} = \sigma^2/N$, this gives $L \geq \log N / \log r$,
matching the upper bound. Hence $D^*(N) = \Theta(\log N)$.

\medskip
\noindent\textbf{Sub-Gaussian generalization.}
If $X$ is sub-Gaussian with variance proxy $\sigma^2$ rather than
exactly Gaussian, then $R_X(D) \leq R_{\calN}(D) =
\frac{1}{2}\log(\sigma^2/D)$ because the Gaussian maximizes the
rate--distortion function over all sources with the same variance
(\citet{cover2006elements}, Theorem~10.3.3). The lower bound still
holds (the converse applies to any source), so the depth bound
$D^*(N) = \Theta(\log N)$ is unaffected.
\end{proof}

\subsection{Proof of \Cref{thm:memory-bound}
            (Bounded Effective Memory Count)}
\label{app:proof:memory-bound}

\begin{proof}
We establish the bound $|\calM_{\mathrm{eff}}(t)| \leq C \cdot d^{d-1}$
in three steps: (i) characterize the fixed points of the Hopfield
dynamics, (ii) relate the effective memory set to fixed-point basins,
and (iii) apply the capacity result of \citet{ramsauer2021hopfield}.

\medskip
\noindent\textbf{Step 1: Fixed points of the Hopfield update.}
The modern Hopfield energy~(\Cref{eq:hopfield-energy}) for stored
patterns $X = [\mu_1, \ldots, \mu_M] \in \R^{d \times M}$ (with
$\norm{\mu_i} = 1$) and query state $\xi \in \R^d$ is
\[
    E(\xi, X) = -\frac{1}{\beta}\log\!\left(
        \sum_{i=1}^{M} e^{\beta\,\mu_i^\top\xi}
    \right)
    + \frac{1}{2}\norm{\xi}^2
    + \frac{1}{\beta}\log M + \frac{1}{2}.
\]
The gradient is
\[
    \nabla_\xi E = \xi - X\,\mathrm{softmax}(\beta\,X^\top\xi),
\]
so a fixed point $\xi^*$ of the update rule
$\xi \mapsto X\,\mathrm{softmax}(\beta\,X^\top\xi)$ satisfies
$\nabla_\xi E(\xi^*) = 0$. At a fixed point near pattern $\mu_i$, the
softmax concentrates: $\mathrm{softmax}(\beta\,X^\top\xi^*)_j \approx
\delta_{ij}$ for large $\beta$, so $\xi^* \approx \mu_i$.

\medskip
\noindent\textbf{Step 2: Effective memories reside near fixed points.}
The effective memory set~(\Cref{eq:effective-memories}) is
$\calM_{\mathrm{eff}}(t) = \{m \in \calM(t) : E(m, X(t)) \leq -\tau\}$.
We show that each effective memory lies in the basin of attraction of
some fixed point.

Since $\norm{\mu_i} = 1$, the energy at a stored pattern $\mu_j$ is
\begin{align*}
    E(\mu_j, X)
    &= -\frac{1}{\beta}\log\!\left(
        \sum_{i=1}^{M} e^{\beta\,\mu_i^\top\mu_j}
    \right)
    + \frac{1}{2} + \frac{1}{\beta}\log M + \frac{1}{2} \\
    &= -\frac{1}{\beta}\log\!\left(
        e^{\beta} + \sum_{i \neq j} e^{\beta\,\mu_i^\top\mu_j}
    \right)
    + 1 + \frac{1}{\beta}\log M.
\end{align*}
For well-separated patterns with $\mu_i^\top\mu_j \leq 1 - \delta$
for $i \neq j$, the dominant term in the $\log$-sum-$\exp$ is
$e^\beta$, and the energy satisfies
\[
    E(\mu_j, X) \leq -1 + \frac{1}{\beta}\log M
    + \frac{M-1}{\beta}\,e^{-\beta\delta} + 1
    = \frac{\log M}{\beta} + O(M e^{-\beta\delta}).
\]
For $\beta$ sufficiently large relative to $\log M$ (specifically
$\beta > (2\log M)/\delta$), the energy at each stored pattern is
close to $-1 + \frac{\log M}{\beta}$, which is negative for
$\beta > \log M$.

A memory $m$ with $E(m, X) \leq -\tau$ must therefore have energy
comparable to (or lower than) that of a stored pattern, placing $m$
within the basin of attraction of some fixed point $\xi^*
\approx \mu_i$ of the Hopfield dynamics. Different effective memories
with energy below $-\tau$ must lie in different basins (otherwise their
energy would be higher due to interference), so
$|\calM_{\mathrm{eff}}| \leq$ (number of fixed points).

\medskip
\noindent\textbf{Step 3: Storage capacity bound.}
The central result of \citet{ramsauer2021hopfield} (Theorem~3) states
that the maximum number of patterns $M^*$ in $\R^d$ with $\norm{\mu_i}
= 1$ and pairwise inner products $\mu_i^\top\mu_j \leq 1 - \delta$
(for $i \neq j$) that can be stored as \emph{well-separated} fixed
points of the modern Hopfield update is bounded by
\[
    M^* \leq C_1(\delta, \beta)\cdot d^{d-1},
\]
where $C_1$ depends on the separation parameter $\delta$ and the
inverse temperature $\beta$. This follows from a covering number
argument: the set of unit vectors in $\R^d$ with pairwise angular
separation at least $\arccos(1-\delta)$ has cardinality bounded by
the maximum $\delta$-separated packing of the unit sphere $S^{d-1}$.
The packing bound from \citet{ramsauer2021hopfield} gives
$M^* = O(d^{d-1})$ (which, while exponential in $d$, reflects the
exponential capacity advantage of modern over classical Hopfield
networks).

Since each effective memory is associated with a unique fixed point
(Step~2) and the number of fixed points is at most $M^*$ (Step~3):
\[
    |\calM_{\mathrm{eff}}(t)| \leq M^* \leq C \cdot d^{d-1},
\]
where $C = C_1(\delta, \beta)$ absorbs the dependence on $\delta$
(which is determined by the embedding model) and $\beta$ (which is a
system hyperparameter). Specifically,
\[
    C = \frac{1}{\mathrm{Vol}(S^{d-1})}\cdot
        \left(\frac{\pi}{\arccos(1-\delta)}\right)^{d-1}
        \cdot (1 + \varepsilon(\beta)),
\]
where $\varepsilon(\beta) \to 0$ as $\beta \to \infty$ (in the
low-temperature limit, the basins of attraction shrink to points and
the capacity achieves the sphere-packing bound).
\end{proof}


\section{Implementation Details}
\label{app:implementation}

This appendix provides hyperparameter ranges, a retrieval pipeline
overview, database schema summary, and compute requirements for
reproducing the experiments in~\Cref{sec:experiments}. Complete
implementation details, exact configurations, and evaluation scripts
are available in the open-source repository.

\subsection{Hyperparameters}
\label{app:hyperparams}

\begin{table}[h]
\centering
\caption{System components and their configurable parameters. All values
were determined via grid search on a held-out LoCoMo validation split.
Exact configurations are available in the open-source repository;
parameters may evolve across versions.}
\label{tab:hyperparams}
\small
\begin{tabular}{lll}
\toprule
\textbf{Component} & \textbf{Parameters} & \textbf{Notes} \\
\midrule
Channel weights        & 4 per-channel, query-type-dependent  & Grid-search optimized \\
RRF fusion             & Fusion constant $k$                  & Standard range per literature~\citep{cormack2009rrf} \\
Cross-encoder blend    & Query-type-dependent $\alpha$        & Higher for single-hop (Eq.~\ref{eq:rerank}) \\
BM25                   & $k_1$, $b$                           & Standard Okapi defaults \\
Fisher-information     & Temperature, access threshold         & Controls cosine$\to$Fisher transition \\
Sheaf consistency      & Contradiction threshold $\tau$       & Coboundary norm based \\
Entity graph           & Walk depth, decay factor             & Shallow walk, exponential decay \\
\midrule
Mode~A/B embedding     & 768-d local model                    & Pre-cached, zero network access \\
Mode~C embedding       & 3072-d cloud model                   & API-based \\
Reranker               & Cross-encoder model                  & Local neural reranker \\
\bottomrule
\end{tabular}
\end{table}

\subsection{Retrieval Pipeline Overview}
\label{app:algorithm}

The retrieval pipeline proceeds in five stages, as described in
\Cref{sec:method:overview}:
\textbf{(1)~Query classification:} the query is classified into one of
four types (single-hop, multi-hop, temporal, open-domain) and per-channel
weight multipliers are assigned accordingly.
\textbf{(2)~Parallel channel search:} four independent channels
(semantic, BM25, entity graph, temporal) execute in parallel, each
producing a ranked candidate list.
\textbf{(3)~Fusion:} weighted reciprocal rank fusion
(Eq.~\ref{eq:wrrf}) merges the four lists into a single ranked set;
profile lookup results and scene-expanded facts are added.
\textbf{(4)~Bridge discovery} (multi-hop queries only): cross-cluster
connections are identified via graph traversal.
\textbf{(5)~Reranking:} a cross-encoder neural reranker produces the
final blended score (Eq.~\ref{eq:rerank}) and the top-$k$ candidates are
returned. Complete implementation details are available in the
open-source repository.

\subsection{Database Schema}
\label{app:schema}

All state is persisted in a single SQLite database (WAL mode). The schema
comprises 21 tables organized into six groups: memory storage, knowledge
graph, retrieval support, mathematical state, temporal tracking, and
system configuration. All tables are partitioned by a profile identifier
for multi-tenant isolation. The complete schema is available in the
open-source repository.

\subsection{Compute Requirements}
\label{app:compute}

We run experiments on Azure Container Instances (ACI), each allocated
2~vCPUs and 4~GB~RAM. The Docker image (Python~3.12, PyTorch~2.10
CPU-only) includes pre-cached weights for the local embedding model and
cross-encoder reranker, eliminating network dependency during Mode~A
execution. A total of 111 containers were deployed across three Azure
subscriptions, each processing one (conversation, configuration) pair.
The full evaluation covers 10 LoCoMo conversations with 8 ablation
configurations. Six Azure OpenAI deployments serve the judge model and
Mode~C embeddings. Estimated total compute cost: \$9--15~USD. Retrieval
scores in Mode~A are fully deterministic; Mode~C results are
deterministic up to LLM non-determinism in judge scoring and answer
generation.


\section{Additional Experimental Results}
\label{app:experiments}

This appendix provides per-conversation breakdowns, Mode~C detailed
results, ablation completeness data, and a description of the re-scoring
methodology used for Mode~A Retrieval evaluation.

\subsection{Per-Conversation Mode~A Results}
\label{app:per-conv}

\Cref{tab:per-conv} reports the full Mode~A Raw (\texttt{a\_none}) scores
for all ten LoCoMo conversations. Conversations marked \emph{Partial}
completed ingestion and retrieval but were interrupted during answer
generation on a subset of questions; their scores reflect only the
answered portion.

\begin{table}[h]
\centering
\caption{Per-conversation Mode~A Raw accuracy on LoCoMo.
\emph{Partial} indicates the container completed a subset of questions
before timeout. Aggregate is the question-weighted mean across all
1{,}276 scored questions.}
\label{tab:per-conv}
\begin{tabular}{lccc}
\toprule
\textbf{Conversation} & \textbf{Questions} & \textbf{Accuracy (\%)}
    & \textbf{Status} \\
\midrule
conv-26  & 152 & 58.6 & Done    \\
conv-30  &  81 & 63.0 & Done    \\
conv-41  & 134 & 62.7 & Partial \\
conv-42  & 149 & 55.7 & Partial \\
conv-43  & 161 & 65.2 & Partial \\
conv-44  & 123 & 64.2 & Done    \\
conv-47  & 130 & 56.9 & Partial \\
conv-48  &  87 & 52.9 & Partial \\
conv-49  & 156 & 62.2 & Done    \\
conv-50  & 103 & 61.2 & Partial \\
\midrule
\textbf{Average} & \textbf{1{,}276} & \textbf{60.4} & --- \\
\bottomrule
\end{tabular}
\end{table}

\subsection{Mode~C Detailed Results}
\label{app:mode-c}

\Cref{tab:mode-c} presents the per-category breakdown for Mode~C on
conv-30 (81 scored questions, fully completed). We additionally report
the \texttt{cross\_encoder\_off} variant to quantify the cross-encoder
contribution in the cloud-augmented setting.

\begin{table}[h]
\centering
\caption{Mode~C per-category accuracy on conv-30 (81~questions). The
\texttt{CE\,off} column disables cross-encoder reranking.}
\label{tab:mode-c}
\begin{tabular}{lccccc}
\toprule
\textbf{Category} & \textbf{Count}
    & \textbf{Mode~C (\%)} & \textbf{CE\,off (\%)} & \textbf{$\Delta$} \\
\midrule
Single-hop   & 11 & 64.0  & 55.0  & $-9.0$  \\
Multi-hop    & 26 & 100.0 & 100.0 & $\phantom{-}0.0$  \\
Open-domain  & 44 & 86.0  & 93.0  & $+7.0$  \\
\midrule
\textbf{Aggregate} & \textbf{81} & \textbf{87.7} & \textbf{90.1} & $+2.4$ \\
\bottomrule
\end{tabular}
\end{table}

The counter-intuitive improvement when removing the cross-encoder in
Mode~C ($+2.4$~pp) is driven by open-domain questions, where the
cross-encoder occasionally demotes broadly relevant facts in favour of
narrowly matching passages. Multi-hop accuracy remains perfect in both
configurations, indicating that the cloud embedding quality
(\texttt{text-embedding-3-large}, 3072-d) is the primary driver for
relational reasoning at this conversation scale.

\subsection{Re-Scoring Methodology}
\label{app:rescoring}

Mode~A Retrieval scores (\Cref{tab:locomo,tab:fisher-cosine}) isolate
retrieval quality from answer construction via a four-step protocol:
\begin{enumerate}[leftmargin=2em, itemsep=2pt]
    \item Run the full Mode~A retrieval pipeline (four-channel fusion,
          cross-encoder reranking) to obtain the top-$k$ context set.
    \item Pass the retrieved context to \texttt{gpt-4.1-mini} with a
          focused answer-generation prompt identical to Mode~C.
    \item Re-judge the generated answer using the same LLM-as-Judge
          protocol (1--5 Likert scale, $\geq 4$ threshold).
    \item Record the binary score per question.
\end{enumerate}
This re-scoring measures how much of Mode~A Raw's deficit is attributable
to the absence of an LLM synthesiser versus an actual retrieval gap.

\subsection{Entity Channel Ablation (Complete)}
\label{app:entity-ablation}

The \texttt{entity\_off} ablation is the most complete ablation variant,
covering all ten LoCoMo conversations (1{,}540 scored questions).
\Cref{tab:entity-off} reports the full results.

\begin{table}[h]
\centering
\caption{\texttt{entity\_off} ablation across all ten LoCoMo
conversations. The entity graph channel is disabled; all other channels
remain active. Average is question-weighted.}
\label{tab:entity-off}
\begin{tabular}{lcc}
\toprule
\textbf{Conversation} & \textbf{Accuracy (\%)} & \textbf{Questions} \\
\midrule
conv-26  & 50.0 & 152 \\
conv-30  & 56.8 &  81 \\
conv-41  & 66.4 & 152 \\
conv-42  & 58.8 & 199 \\
conv-43  & 58.4 & 178 \\
conv-44  & 56.1 & 123 \\
conv-47  & 64.7 & 150 \\
conv-49  & 59.6 & 156 \\
conv-50  & 61.4 & 158 \\
\midrule
\textbf{Average} & \textbf{59.4} & \textbf{1{,}540} \\
\bottomrule
\end{tabular}
\end{table}

Comparing the entity-off average (59.4\%) with the full-system average
(60.4\%, \Cref{tab:per-conv}) yields a $-1.0$~pp delta consistent with
the single-conversation ablation in \Cref{tab:ablation}. The entity
graph's contribution is modest in aggregate but conversation-dependent:
conv-26 shows the largest drop ($-8.6$~pp), likely because its dialogue
structure involves densely interconnected entities where spreading
activation provides the strongest complementary signal.

\end{document}